\documentclass[a4paper,fleqn]{cas-sc}

\usepackage[authoryear]{natbib}

\usepackage{lipsum}
\usepackage{amsfonts}
\usepackage{graphicx}
\usepackage{epstopdf}
\usepackage{algorithmic}
\usepackage{algorithm}
\usepackage[utf8x]{inputenc}
\usepackage{bm}
\usepackage{booktabs}
\usepackage{etoolbox}
\usepackage{todonotes}
\usepackage{siunitx}
\usepackage{amsopn}
\usepackage{xcolor}

\newcommand{\IR}{{\mathbb R}}
\newcommand{\R}{{\mathbb R}}

\newcommand{\beq}{\begin{equation}}
\newcommand{\eeq}{\end{equation}}

\newcommand{\Pb}{\mathbb P}

\newcommand{\E}{\mathbb{E}}

\newcommand{\hac}{\mathcal{H}}

\newcommand{\I}{{\mathbb I}}
\newcommand{\aac}{\mathcal{A}}

\newcommand{\jac}{\mathcal{J}}

\newtheorem{theorem}{Theorem}

\newdefinition{remark}{Remark}
\newdefinition{hypothesis}{Hypothesis}
\newdefinition{fact}{Fact}
\newproof{proof}{Proof}

\newdefinition{definition}{Definition}
\newdefinition{example}{Example}

\begin{document}

\let\WriteBookmarks\relax
\def\floatpagepagefraction{1}
\def\textpagefraction{.001}

\shorttitle{Parameter Estimation in SDEs via WCE and SGD}
\shortauthors{F. Delgado-Vences et~al.}

\title [mode = title]{Parameter Estimation in Stochastic Differential Equations via Wiener Chaos Expansion and Stochastic Gradient Descent}

\author[1]{Francisco Delgado-Vences}
\cormark[1]
\ead{FranciscoJavier.Delgado@uab.cat}
\affiliation[1]{organization={Universitat Autónoma de Barcelona, Departament de Matemàtiques, Edifici C, Facultat de Ciències},
                postcode={080193}, 
                city={Bellaterra}, 
                country={Spain}}

\author[2]{Jose Julian Pavon-Espa\~{n}ol}
\ead{julian.pavon2@ciencias.unam.mx}
\affiliation[2]{organization={Facultad de Ciencias - Universidad Nacional Aut\'onoma de M\'exico. Ciudad Universitaria},
                city={Ciudad de M\'exico}, 
                country={Mexico}}

\author[3]{Arelly Ornelas}
\ead{aornelasv@ipn.mx}
\affiliation[3]{organization={Secihti - Centro Interdisciplinario de Ciencias Marinas. IPN.},
                state={BCS},
                city={La Paz}, 
                country={México}}

\cortext[cor1]{Corresponding author}
\fntext[fn1]{FDV is supported by Ministerio de Ciencia e Innovación, Gobierno de España, project ref. PID2021-123733NB-I00.}

\begin{abstract}
This study addresses the inverse problem of parameter estimation for Stochastic Differential Equations (SDEs) by minimizing a regularized discrepancy functional via Stochastic Gradient Descent (SGD). To achieve computational efficiency, we leverage the Wiener Chaos Expansion (WCE), a spectral decomposition technique that projects the stochastic solution onto an orthogonal basis of Hermite polynomials. This transformation effectively maps the stochastic dynamics into a hierarchical system of deterministic functions, termed the \textit{propagator}. By reducing the stochastic inference task to a deterministic optimization problem, our framework circumvents the heavy computational burden and sampling requirements of traditional simulation-based methods like MCMC or MLE. The robustness and scalability of the proposed approach are demonstrated through numerical experiments on various non-linear SDEs, including models for individual biological growth. Results show that the WCE-SGD framework provides accurate parameter recovery even from discrete, noisy observations, offering a significant paradigm shift in the efficient modeling of complex stochastic systems.
\end{abstract}


\maketitle

\section{Introduction}
Mathematical optimization is fundamental to applied mathematics, influencing areas like optimal control theory, inverse problems, and the training of modern deep neural networks. The main aim is usually described as the minimization of a continuously differentiable objective function $f: \mathbb{R}^d \to \mathbb{R}$. While some problems can be solved analytically, the overwhelming majority of practical scenarios demand iterative numerical approaches. Among these, first-order methods—specifically Gradient Descent (GD) and its stochastic counterpart, Stochastic Gradient Descent (SGD)— are very important due to their implementation simplicity and remarkable scalability in high-dimensional spaces. 

This research focuses on the intersection of these optimization techniques with the theory of Stochastic Differential Equations (SDEs). Specifically, we address the inverse problem of parameter estimation in SDEs by reformulating it as a deterministic optimization task via the Wiener Chaos Expansion.

\subsection{Deterministic Gradient Descent}
 The concept of gradient-based optimization originates from Cauchy's steepest descent method for simultaneous equations. In the context of unconstrained minimization, the Gradient Descent algorithm creates a sequence of iterates represented as $\{x_k\}_{k\geq0}$ through the specified update rule:
\begin{equation} \label{eq:gd_update}
    x_{k+1} = x_k - \eta_k \nabla f(x_k),
\end{equation}
where $\eta_k > 0$ denotes the step size (or learning rate). The theoretical convergence of this method is well-established; for convex functions with Lipschitz continuous gradients ($L$-smooth), GD converges at a rate of $O(1/k)$ (see \cite{Beck2017}, \cite{Nesterov1983}).

However, the efficiency of standard GD is limited by the condition number of the Hessian matrix. To mitigate this issue, Polyak introduced the heavy-ball method, or momentum, using the history of update vectors to accelerate convergence. Further theoretical breakthroughs were made by Nesterov, who proved that by employing a specific momentum schedule, one can achieve a convergence rate of $O(1/k^2)$, which is optimal for the class of first-order methods on smooth convex functions. Despite these advances, the computational cost of GD becomes prohibitive in large-scale settings where the objective function decomposes into a finite sum, $f(x) = \frac{1}{n} \sum_{i=1}^n f_i(x)$, requiring $n$ gradient evaluations per iteration.

\subsection{Stochastic Gradient Descent}
To tackle the computational challenges inherent in full-batch gradient calculation, \cite{Robbins1951} introduced the Stochastic Approximation method.  When applied to the optimization process, this concept results in Stochastic Gradient Descent (SGD). Rather than computing the exact gradient, SGD provides an estimate by utilizing a single sample (or a small mini-batch) $i_k$ chosen uniformly at random:
\begin{equation} \label{eq:sgd_update}
    x_{k+1} = x_k - \eta_k \nabla f_{i_k}(x_k).
\end{equation}
It is essential to note that the stochastic gradient serves as an unbiased estimator of the true gradient, satisfying $\mathbb{E}[\nabla f_{i_k}(x_k)] = \nabla f(x_k)$.

The convergence analysis of SGD differs fundamentally from GD due to the introduction of variance. While GD converges to the exact minimizer with a constant step size, SGD will eventually oscillate around the minimizer within a "noise ball" proportional to the variance of the gradients and the step size (\cite{Bottou2010}, \cite{Bottou2018}). To guarantee almost sure convergence to the optimum, the step size sequence must satisfy the Robbins-Monro conditions:
\begin{equation}
    \sum_{k=1}^{\infty} \eta_k = \infty \quad \text{and} \quad \sum_{k=1}^{\infty} \eta_k^2 < \infty.
\end{equation}
Under these conditions, SGD achieves a sublinear convergence rate of $O(1/k)$ for strongly convex functions. While asymptotically slower than GD, the per-iteration cost of SGD is independent of the dataset size $n$, rendering it superior for large-scale problems.

\subsection{Comparison with Variational Data Assimilation Methods}
\label{subsec:comparison_variational}

In the context of parameter estimation for Stochastic Differential Equations (SDEs), it is pertinent to contrast our proposed framework with widely adopted variational data assimilation methods, such as weak-constraint 4D-Var and stochastic optimal control \cite{tremolet2006accounting, reich2015probabilistic}. Variational approaches typically retain the stochastic nature of the system, formulating the inverse problem by treating the intrinsic process noise as an additional control variable. Under this paradigm, finding the optimal parameters and state trajectories requires minimizing a cost functional that balances the discrepancy with the observations and the magnitude of the model error. However, performing this optimization generally relies on the continuous adjoint method, which entails computationally demanding forward and backward time integrations of the stochastic dynamics \cite{stuart2010inverse}.

In contrast, the methodology presented in this study fundamentally circumvents the need for backward stochastic integration. By leveraging the Wiener Chaos Expansion (WCE), we project the stochastic solution onto a basis of Hermite polynomials, effectively determinizing the system. As previously described, the stochastic dynamics are decomposed into a hierarchical system of deterministic ordinary differential equations, referred to as the propagator. Consequently, the stochastic parameter learning task is rigorously reduced to a deterministic optimization problem. This spectral transformation allows us to compute gradients analytically and apply Stochastic Gradient Descent (SGD) directly on the propagator. By doing so, our framework avoids the severe computational burden and the complex non-convexity landscapes often associated with traditional simulation-based or adjoint-based variational methods, while ensuring robust convergence properties for the parameter estimation.

\subsection{The Wiener Chaos Expansion}

The Wiener Chaos Expansion (WCE) is a spectral decomposition technique on the probability space, originally introduced by Cameron and Martin (1947) and later developed for SDEs by Rozovskii and Lototsky and their colaborators among others. It provides an orthogonal basis for the space of square-integrable random variables using Hermite polynomials of Gaussian processes.

In the context of this work, the WCE serves a dual purpose. First, it allows for the rigorous existence and uniqueness analysis of solutions to Stochastic Partial Differential Equations (SPDEs). Second, and most important for our method, it provides a computational framework to decompose a stochastic process into a hierarchical system of deterministic differential equations, known as the \textit{propagator}. This transformation enables us to convert the stochastic inverse problem of parameter estimation into a deterministic optimization problem, solvable via gradient-based methods.

\subsection{Overview of the Proposed Methodology}
\label{subsec:overview_methodology}

While the preceding sections outline the foundational mathematical concepts, the primary contribution of this manuscript is the synthesis of these distinct tools into a unified, highly efficient framework for SDE parameter estimation. Specifically, we address the challenge of inferring the unknown drift and diffusion parameters, denoted by $\theta$, from discrete noisy measurements. This is a critical task when modeling highly non-linear dynamics, such as those governing individual biological growth.

Traditional inference frameworks for SDEs, such as Maximum Likelihood Estimation (MLE) and Markov Chain Monte Carlo (MCMC) methods, often rely on intensive Monte Carlo simulations to approximate the transition densities or the likelihood surface. These approaches become computationally prohibitive as the complexity of the system or the dimension of the parameter space increases. In contrast, our framework leverages the Wiener Chaos Expansion to transform the stochastic inverse problem into a hierarchical system of deterministic equations. By doing so, we eliminate the need for repetitive path simulations, providing a more efficient optimization landscape for Stochastic Gradient Descent (SGD) to operate on.

Instead of relying on standard maximum likelihood estimators or particle-based filters, our approach reformulates the problem through three fundamental steps:

\begin{enumerate}
    \item \textbf{Spectral Determinization:} Rather than dealing with random trajectories directly, we apply the Wiener Chaos Expansion to the underlying SDE. By projecting the stochastic process onto an orthogonal basis of Hermite polynomials, we truncate and transform the original stochastic system into a coupled, hierarchical system of deterministic Ordinary Differential Equations (ODEs). The resulting deterministic system is what we refer to as the \textit{propagator}.
    
    \item \textbf{Formulation of the Discrepancy Functional:} We cast the parameter inference as a deterministically driven inverse problem. We construct a Tikhonov-Phillips regularized cost functional that quantifies the error between the statistical properties of the experimental data (specifically, the empirical mean and variance) and the corresponding zeroth and higher-order states of our deterministic propagator. 
    
    \item \textbf{Optimization via SGD:} To minimize this functional and recover the optimal parameters $\theta$, we analytically compute the gradients directly from the propagator equations. We then deploy a Stochastic Gradient Descent algorithm to iteratively navigate the optimization landscape.
\end{enumerate}

By executing this sequence, we effectively translate a complex stochastic inference problem into a classical machine learning optimization task. This paradigm shift not only avoids the computational bottlenecks of Monte Carlo simulations but also leverages the robust convergence properties of gradient-based methods on deterministic, regularized loss landscapes.

The remainder of this manuscript is organized as follows. In Section \ref{sec:InvProb} we briefly review the inverse problem and in particular the ill-posed inverse problems. Section \ref{DO-noise} formalizes the theoretical framework for differential operators subjected to noise, establishing the stochastic inverse problem as a well-posed minimization task grounded in Tikhonov-Phillips regularization. Section \ref{examples} bridges this abstract theory with classical models, demonstrating the construction of the deterministic propagator and the corresponding objective functions for both SDEs and Stochastic Partial Differential Equations (SPDEs). Section \ref{sec:LearSDEs} details the core algorithmic contribution: it introduces a modified empirical loss function incorporating the quadratic variation of the data, explicitly computes the analytical gradients from the propagator, and discusses the convergence guarantees of the SGD algorithm utilizing the Barzilai-Borwein step-size method. Section \ref{sec:NumExp} validates the proposed framework through numerical experiments on a foundational toy model, the Ornstein-Uhlenbeck process. In Section \ref{sec:AppExamSDEs} we apply and validate the proposed method to three examples of SDEs. Finally, Section \ref{real-data-app} extends the methodology to more complex, non-linear SDEs with practical relevance, including Geometric Brownian Motion, and the stochastic logistic equation for biological growth.

\section{On inverse problems}\label{sec:InvProb}

\subsection{Inverse Problems and Regularization}
A prominent domain where these iterative methods have found fundamental application is the regularization of ill-posed inverse problems. In this setting, the objective is to recover an unknown parameter $x \in \mathcal{G}$ from noisy measurements $y^\delta \in \mathcal{H}$, modeled by the operator equation:
\begin{equation} \label{eq:inverse_prob}
    \mathcal{A}x = y^\delta ,
\end{equation}

where $\mathcal{A}: \mathcal{G} \to \mathcal{H}$ is a forward operator (potentially non-linear). Following Hadamard's definition, such problems are often \textit{ill-posed}, meaning a solution may not exist, may not be unique, or may not depend continuously on the data.

The classical approach to solving ill-posed problems involves minimizing a discrepancy functional, typically the least-squares term $\frac{1}{2}\|\mathcal{A}x - y^\delta\|^2$. To address instability, this minimization is often augmented with a regularization term. 
\begin{equation}\label{min_squares_regular}
 \,\Big( \frac{1}{2}\big\| y- Au \big\|_Y^2 + \lambda \big\| u - x_0 \big\|_E^2 \Big).
\end{equation}
The \textit{Tikhonov-Phillips regularization} searches for a solution $x_\delta$ that minimizes:
\begin{equation} \label{eq:tikhonov}
    J_\delta(x) = \|\mathcal{A}x - y^\delta\|_{\mathcal{H}}^2 + \delta \|x - x_0\|_{\mathcal{G}}^2,
\end{equation}
where $\delta > 0$ is the regularization parameter and $x_0$ is a prior estimate (often set to zero).

It is well-established that applying Gradient Descent to this functional yields the celebrated \textit{Landweber iteration} \cite{Landweber1951}, a method known to exhibit a self-regularizing property where the iteration count serves as the regularization parameter (semiconvergence) \cite{Engl1996}. However, for large-scale tomographic or seismic imaging problems where the operator $A$ consists of distinct row blocks representing partial data subsets, computing the full gradient is computationally prohibitive. Here, Stochastic Gradient Descent emerges naturally as the \textit{Randomized Kaczmarz method} (or Algebraic Reconstruction Technique) \cite{Strohmer2009}. By selecting random subsets of equations to update the estimate, SGD acts as an efficient iterative solver that drastically reduces the per-iteration computational cost while maintaining convergence guarantees suitable for large-scale inverse problems \cite{Needell2014}.

\subsection{Some statistical Ill-posed inverse problems}\label{S-IPIP}

In this subsection and for the sake of completeness, we will review some classical inverse problems. There is a high number of good references for the study of the inverse problems, here we refer to: \cite{lu2013regularization}, \cite{kabanikhin2011inverse},  \cite{engl1996regularization}, \cite{hasanouglu2021introduction}, \cite{isakov2006inverse} and the references therein.\\
Set two separable Hilbert spaces $\hac$ and $\mathcal{G}$, and  define an operator $A$ from  $\mathcal{G}$ into $\hac$, $A:\mathcal{G}\rightarrow \hac$. 

We consider the problem 
\begin{equation}\label{Gen_Inv-problem}
A x=y
\end{equation}

with $y\in \hac$ and $x\in \mathcal{G}$. Here, we are assuming that $A$ is a (not necessarily bounded) linear operator acting from a subset $X$ of $\mathcal{G}$ into a subset $Y$ of $\mathcal{H}$. The inverse problem focuses on finding $x \in X$ given $y$; sometimes $y$ is called the {\it data}.\\
\begin{remark}
In this work, we will be interested in the case where the operator  $A$ is a differential operator.
\end{remark}

The problem (\ref{Gen_Inv-problem}) is called {\it well-posed (in the sense of Hadamard)} on $\hac$ and $\mathcal{G}$
if the following three conditions hold:
\begin{enumerate}
\item for any $y \in \mathcal{H}$ there exists a solution $x_e \in \mathcal{G}$ to the equation (\ref{Gen_Inv-problem}), i.e., $Range (A)=\mathcal{H}$ (the existence condition);
\item the solution $x_e$ to the equation (\ref{Gen_Inv-problem}) is unique in $\mathcal{G}$ (the uniqueness condition);
\item for any neighbourhood $\mathcal{O}\left(x_e\right) \subset \mathcal{G}$ of the solution $x_e$ to the equation (\ref{Gen_Inv-problem}), there is a neighbourhood $\mathcal{O}(y) \subset \mathcal{H}$ of  $y$ such that for all $y_\delta \in \mathcal{O}(y)$ the element $A^{-1} y_\delta=x_\delta$ belongs to the neighbourhood $\mathcal{O}\left(x_e\right)$, i.e., the operator $A^{-1}$ is continuous (the stability condition). \end{enumerate}
There are three causes leading to ill-posedness of the equation (\ref{Gen_Inv-problem}): 
\begin{itemize}
\item[(i)] $Range(A)= \overline{Range(A)}\ne \mathcal{H}$.
\item[(ii)]  $kernel(A)\ne \{0\}$
\item[(iii)] The range of $A$ is not closed in $\mathcal{H}$.
\end{itemize}

For some discussion on the issues and the possible solutions to each cause see some of the references mentioned before. 

The minimization problem   (\ref{eq:tikhonov}) is an example of what is called a Tikhonov-Phillips regularizer (cf. \cite{kaipio2006statistical}).

\begin{definition}
Let $\delta>0$ be a given constant. The Tikhonov–Phillips regularized solution $x_\delta \in \mathcal{H}$ is the minimizer of the functional
$$
F_\delta(x)=\|A x-y\|_Y^2+\delta\|x\|_X^2,
$$
provided that a minimizer exists. The parameter $\delta>0$ is called the regularization parameter.
\end{definition}
\begin{remark}
    Note that in (\ref{min_squares_regular})  by setting $x_0 = 0$, the problem simplifies to finding the solution with the minimum norm (or energy) that fits the data. This is the canonical form of Tikhonov regularization used when no specific prior estimate $x_0$ is available.
\end{remark}

\section{Differential operators with a noise}\label{DO-noise}

\subsection{Framework} \label{Malliavin-section}

Denote by $\hac$ a separable Hilbert space, finite or infinite dimensional.

Let $(\Omega, \mathcal{F}, \Pb; \hac)$ be a Gaussian  probability space. That means that $\hac$ is a closed subspace of $L^2(\Omega, \mathcal{F}, \Pb)$ whose elements are zero-mean Gaussian random variables; in such manner that we could define a isonormal Gaussian process $W=\big( W(h): h\in \hac\big)$. That means $W$ is a centered Gaussian family
of random variables such that $\E(W (h)W (g)) = \langle h,g\rangle_{\hac}$ for all $h, g \in \hac$, (cf. \cite{nualart2006malliavin}). \\
 We will be interested in the white noise case. That is, when we could interpret the elements of $\mathcal{H}$ as  stochastic integrals of functions in $L^2(T, \mathcal{B}(T), \mu)$ with respect to a Gaussian measure on the parameter space $T$, here $\mu$ is a $\sigma$-finite measure without atoms (cf. \cite{nualart2006malliavin}). Thus, by taking $\hac$ finite or infinite dimensional, we could study SDEs and SPDEs using this framework.

We define the Hilbert space $\mathbb{H}=L^2(\Omega;\hac)$, with norm given by 
$$
\|x\|_{\mathbb{H}}^2 := \E\big(\|x\|_{\mathcal{H}}^2\big)
$$

In this section, we consider the stochastic extension of the inverse problem (1.5), formulated as the operator equation:
\begin{equation} \label{Sto_Inv-problem}
    \mathcal{A}x = y + \mathcal{B}\eta,
\end{equation}
where $y \in \mathcal{H}$ represents the observation (or the mean data), and $\mathcal{B}\eta$ represents the stochastic perturbation, with $\eta$ being an isonormal process representing observational noise.

\textbf{On the Solution:}
Since the source term involves the random variable $\eta$, the solution $x$ is inherently stochastic. Therefore, we seek a solution in the space of square-integrable $\mathcal{G}$-valued random variables, denoted by $\mathbb{G} = L^2(\Omega, \mathcal{F}, \mathbb{P}; \mathcal{G})$. The equality in \eqref{Sto_Inv-problem} is understood in the mean-square sense; however, due to the ill-posed nature of the operator $\mathcal{A}$, an exact solution satisfying the equation almost surely may not exist. Consequently, we define the \textit{generalized solution} to \eqref{Sto_Inv-problem} as the unique minimizer of the regularized expected discrepancy functional:
\begin{equation}
    x_{sol} := \operatorname*{argmin}_{u \in \mathbb{G}} \frac{1}{2} \| \mathcal{A}u - (y + \mathcal{B}\eta) \|_{\mathbb{H}}^2.
\end{equation}
This formulation naturally leads to the stochastic optimization problem (2.4), where the solution is an estimator characterized by its Wiener Chaos expansion.

We will assume the following set of hypotheses.

 \begin{hypothesis}\label{On_operators}
\begin{itemize}
    \item[a)] The operator $\aac$ takes values in $\hac$.
    \item[b)] The mean value of $\aac x$ is equal to $A x$ from the deterministic problem (\ref{Gen_Inv-problem}).
\end{itemize}
 
 \end{hypothesis}

\begin{remark}
\begin{itemize}
    \item  Observe that given the stochasticity of $\eta$ in Eq. (\ref{Sto_Inv-problem}), thus the value $x$ that solve the inverse problem is stochastic too.  Therefore, we need to consider a solution of the inverse problem that lives in a suitable space of random variables (see Da Prato and Zabczyk  \cite{da-za1} for instance).
\item The assumption that the first moment $(\mathcal{A}u)_0$ coincides with the deterministic operator $Ax$ relies on the property of the It\^o integral, which vanishes in expectation. However, strictly speaking, this equality holds exactly only for linear drift terms, meaning for the case in which  $A$ is a linear operator. For non-linear systems, such as the logistic equation, Jensen's inequality implies that $\mathbb{E}[\mu(X_t)] \neq \mu(\mathbb{E}[X_t])$. In such cases, minimizing the misfit $\|(\mathcal{A}u)_0 - y\|^2$ approximates fitting the data to the projected mean dynamics of the system
\end{itemize}
\end{remark}

The kind of problems like Eq. (\ref{Sto_Inv-problem}) cover examples such as stochastic differential equations and stochastic partial differential equations.
Taking expectation in  Eq. (\ref{Sto_Inv-problem})  we have  $\E\big( \mathcal{A} x\big)=y$.

Thus the stochastic problem has as a particular case the deterministic problem (\ref{Gen_Inv-problem}).\\
Using the notation and definitions from appendix, we will assume the following set of hypotheses.

 \begin{hypothesis}
\begin{enumerate}
    \item We have the Wiener-chaos expansion for  $\aac x$ :
$$
\aac x = \sum_{\alpha \in \jac } (\aac x)_\alpha\,  \xi_\alpha\qquad\mbox{with } (\aac x)_\alpha := \E\big[ (\aac x )\, \xi_\alpha\big].
$$
    
\end{enumerate}
 
 \end{hypothesis}

Using the facts from the Remark \ref{1st-2nd_moments} jointly with the Hypotheses \ref{On_operators}-b), we have that 
\begin{equation}
    \mathbb{E}(\mathcal{A} x) =  \sum_{\alpha \in \mathcal{J} } (\mathcal{A} x)_{\alpha}\, \I_{|\alpha|=0} =: (\mathcal{A} x)_{0}\label{mean_WCE}
\end{equation}

\subsection{Stochastic inverse problem}
Consider  the  optimization problem for the stochastic equation (\ref{Sto_Inv-problem}):
\begin{equation}\label{min_squares_A}
v(x)=\textrm{argmin}_{x\in X}\, \frac{1}{2}\big\| y- \mathcal{A}x \big\|_\mathbb{H}^2. 
\end{equation}

We have the main result of this section.
\begin{theorem}\label{tikhonov}
Assume hypotheses 1 and 2 are satisfied. Thus, the inverse problem (\ref{min_squares_A}) is well-posed. 
\end{theorem}

\begin{proof}
We will demonstrate that the stochastic minimization problem (\ref{min_squares_A}) is mathematically equivalent to a Tikhonov-Phillips regularization for a deterministic problem, which guarantees its well-posedness.

Let $\mathbb{H}$ denote the appropriate space of square-integrable random variables with values in $\mathcal{H}$. Using the Wiener Chaos Expansion, the discrepancy functional can be expanded as follows:
\begin{align}\label{tech_1}
 \big\| \mathcal{A}x - y \big\|_\mathbb{H}^2 &= \E\Big[\big\| \mathcal{A}x - y \big\|_\mathcal{H}^2 \Big] = \E\Bigg[\Big\| \sum_{\alpha \in \jac} (\aac x)_\alpha \xi_\alpha - y \Big\|_\mathcal{H}^2 \Bigg] \nonumber \\
 &= \E\Bigg[\Bigg\| \Big( (\aac x)_{\mathbf{0}} - y \Big) + \sum_{\alpha \in \jac\,: |\alpha|\ge 1 } (\aac x)_\alpha \xi_\alpha \Bigg\|_\mathcal{H}^2 \Bigg],
\end{align}
where we have used the fact that for the multi-index $|\alpha| = 0$, the basis polynomial is $\xi_{\mathbf{0}} \equiv 1$.

Expanding the square of the norm yields three terms:
\begin{align}\label{tech_2}
 \big\| \mathcal{A}x - y \big\|_\mathbb{H}^2 &= \Big\| (\aac x)_{\mathbf{0}} - y \Big\|_\mathcal{H}^2 + \E\Bigg[\Big\| \sum_{\alpha \in \jac\,: |\alpha|\ge 1 } (\aac x)_\alpha \xi_\alpha \Big\|_\mathcal{H}^2 \Bigg]  + 2 \sum_{\alpha \in \jac\,: |\alpha|\ge 1 } \big\langle (\aac x)_{\mathbf{0}} - y , (\aac x)_\alpha \big\rangle_\mathcal{H} \, \E[\xi_\alpha].
\end{align}
By the properties of the WCE basis, the random variables $\xi_\alpha$ are centered for all $|\alpha| \ge 1$, meaning $\E[\xi_\alpha] = 0$. Consequently, the cross-term vanishes completely. 

Furthermore, due to the orthogonality of the basis $\{\xi_\alpha\}_{\alpha \in \jac}$, we have $\E[\xi_\alpha \xi_\beta] = \delta_{\alpha \beta}$. Applying this isometry to the variance term, the expectation of the squared norm simplifies to the sum of the squared norms. Assuming the basis is normalized such that $\E[\xi_\alpha^2] = 1$, we obtain:
\begin{equation}\label{tech_3}
 \big\| \mathcal{A}x - y \big\|_\mathbb{H}^2 = \underbrace{\Big\| (\aac x)_{\mathbf{0}} - y \Big\|_\mathcal{H}^2}_{\text{Data Fidelity}} \, + \, \underbrace{\sum_{\alpha \in \jac\,: |\alpha|\ge 1 } \Big\| (\aac x)_\alpha \Big\|_\mathcal{H}^2}_{\text{Regularization Penalty}}.
\end{equation}

The right-hand side of Equation (\ref{tech_3}) takes the exact form of a Tikhonov-Phillips regularized functional for the deterministic forward operator $x \mapsto (\mathcal{A}x)_{\mathbf{0}}$. The first term enforces data fidelity by matching the mean dynamics to the deterministic observations $y$, while the infinite sum of higher-order chaos coefficients acts as a strongly convex regularization penalty that minimizes the stochastic variance (or energy) of the system. By invoking classical regularization theory for deterministic inverse problems, this structure ensures the existence, uniqueness, and stability of the minimizer, concluding the proof.
\end{proof}

While the isometric properties of the Wiener Chaos Expansion are standard in stochastic analysis, Theorem \ref{tikhonov} explicitly bridges this spectral property with the deterministic regularization theory of inverse problems, demonstrating that the higher-order stochastic fluctuations naturally induce a Tikhonov-Phillips penalty. \\

Consequently, this theorem serves as the theoretical cornerstone of our methodology. By formally establishing that the spectral projection of the stochastic system inherently regularizes the inverse problem, we justify the transition from a computationally prohibitive stochastic optimization to a tractable deterministic one. Optimizing the resulting deterministic propagator via Stochastic Gradient Descent (SGD) is therefore not merely a heuristic for computational efficiency, but a mathematically rigorous strategy. It guarantees that the parameter estimation remains well-posed and remarkably stable against observational noise, effectively circumventing the severe instability and exhaustive sampling requirements typically associated with direct simulation-based inference.
\subsection{Well-posedness of the Minimization Problem}

The minimization of the functional defined in \eqref{min_squares_A} constitutes a variational approach to the inverse problem. To establish the well-posedness of this formulation—specifically the existence, uniqueness, and stability of the minimizer—we draw upon recent advances in Bayesian inverse problems in Banach spaces.

As demonstrated by \cite{Chen2024}, the regularized inverse problem can be interpreted as determining the mode Maximum A Posteriori (MAP estimator) of a posterior measure. They established that under general Banach space priors, the posterior distribution satisfies a Lipschitz continuity condition with respect to the observed data in the Hellinger metric. This result is crucial for our framework, as the variance term in (\ref{tech_3}) , $\sum_{|\alpha| \ge 1} \|(\mathcal{A}u)_\alpha\|^2$, acts as a strongly convex regularizer in the probability space.

Furthermore, Chen et al. proved that such formulations exhibit \textit{posterior contraction}, meaning that as the observational noise vanishes (or the amount of data increases), the probability mass of the solution concentrates around the true parameter value. This theoretical guarantee justifies the use of Equation (\ref{tech_3})  not merely as a heuristic, but as a well-posed mathematical problem where the stochastic "variance penalty" ensures the stability of the reconstruction even when the forward operator is non-linear.
  
\section{Classical examples of differential operators}\label{examples}

\subsection{Stochastic differential equations}

Suppose that a process $X_t$  is given by the stochastic differential equation (SDE)
 \begin{eqnarray}\label{SDE1}
 dX(t)&=& f(X_t)dt + \sigma(X(t))dB_t, \\
X(0)&=& x_0.\nonumber
\end{eqnarray}

 We assume that the functions $f$ and $\sigma$ satisfy suitable assumptions to ensure the existence and uniqueness of a solution (see \cite{oksendal2013stochastic} for instance). 
In addition, since we are looking to estimate some parameters from the SDE  (\ref{SDE1}) then we will make explicit the dependence of the functions $f$  and $\sigma$ on some parameters. More precisely, we suppose that $f(x)=f(x;\theta_1)$  and $\sigma(x)=\sigma(x,\theta_2)$. Here,  $(\theta_1,\theta_2)$, are the parameters involved for the inverse problem.

The SDE (\ref{SDE1}) corresponds to the case where the operator $\mathcal{A}$ is an  ordinary differential operator.

We assume the following.

 \begin{hypothesis}\label{hypo1}
 The functions $f$ and $\sigma$  can we written as 
$$ f(x;\Theta)= F(x)g_1(\theta_1);\quad  \sigma(x;\Theta)= \Gamma(x)g_2(\theta_2),$$ with $g_1,g_2$ being differentiable functions with respect to the parameters $\theta_1$ and $\theta_2$, respectively. 
  \end{hypothesis}
 Note that since we are interested in estimating the parameters $\theta_1$ and $\theta_2$ we assume that the functions $g_1(\theta_1)$ and $g_2(\theta_2)$ are independent of time and deterministic.

\subsubsection{Wiener Chaos expansion for the SDE} \label{WienerChaos-BM}

In this section, we present a procedure to obtain the propagator for the SDE (\ref{SDE1}). 
We know that the solution $X$ belongs to the space $L^2(\Omega; [0,T])$ then, we can write the solution as in (\ref{eq:WCE}):
\begin{equation}\label{chaos-X}
X(t)=\sum_{m\in\mathcal{J}} X_m(t) \xi_m 
\end{equation}

where $ X_m(t)=\E\big( X(t) \xi_m \big) $, and we are using the notation from the appendix. Thus, we have

\begin{align}\label{xm1}
\E\bigg[ \Big( X(t) \Big)\xi_m \bigg] &=  \E\bigg[ \Big( x_0 + \int_0^t f(X(s)) ds + \int_0^t \sigma(X(s)) dB(s) \Big) \xi_m \bigg]\nonumber\\
 &= \E\bigg[ x_0 \xi_m \bigg] +\E\bigg[ \Bigg( \int_0^t f(X(s)) ds \Bigg)\xi_m \bigg] + \E\bigg[ \Bigg(\int_0^t \sigma(X(s)) dB(s) \Bigg) \xi_m \bigg]\nonumber\\
 &=  x_0 \I_{|m|=0} + \int_0^t g_1(\theta_1) \E\Big[ F(X(s)) \xi_m\Big] ds +   \E\bigg[ \Bigg(\int_0^t \sigma(X(s)) dB(s) \Bigg) \xi_m \bigg]\nonumber\\
 &= x_0 \I_{|m|=0} +  g_1(\theta_1) \int_0^t  F_m(X(s))ds  +  g_2(\theta_2) \E\bigg[ \Bigg(\int_0^t \Gamma(X(s)) dB(s) \Bigg) \xi_m \bigg].
\end{align}

For the last integral in (\ref{xm1}) we will use the results from the appendix and Hypothesis \ref{hypo1}. Indeed, we know that the integrand is adapted to the generated filtration $\mathcal{F}_t$ generated by the noise $B(t), 0\le t\le T$. Thus, the Skorohod integral is reduced to the usual It\^o stochastic integral. In other words, we can write the following
$$
 \int_0^t \Gamma(X(s)) dB_s = \delta\Big(\Gamma(X(\cdot))\I_{[0,t]}(\cdot)  \Big)
$$
Thus, using this expression and (\ref{def-delta}) we get
\begin{align}\label{A-term}
\E\bigg[ \Big( \int_0^t\Gamma(X(s)) dB_s \Big) \xi_m \bigg] &= \E\Bigg( \delta\bigg(\Gamma(X(\cdot)) \I_{[0,t]}(\cdot)   \bigg) \xi_m \Bigg) = \E\Big(\big\langle \Gamma(X(\cdot))  \I_{[0,t]}(\cdot)  , D \xi_m \big\rangle_{\mathbb{H}} \Big) \nonumber\\
&= \E\bigg[  \int_0^t  \Gamma(X(s)) D \xi_m  ds  \bigg] = \sum_{i=1}^\infty \sqrt{m_i} \int_0^t e_i(s) \E\Big[ \xi_{m^-(i)} \sigma(X(s))   \Big] ds \nonumber\\
&= \sum_{i=1}^\infty \sqrt{m_i}    \int_0^t e_i(s) \Gamma_{m^-(i)}(X(s))    ds
\end{align}

Plugging together (\ref{xm1}) - (\ref{A-term}) we have

\begin{align}\label{xm-equation}
X_m(t) &= x_0\I_{\{|m|= 0\} }  +   g_1(\theta_1) \int_0^t  F_m(X(s))ds  +  g_2(\theta_2) \sum_{i=1}^\infty \sqrt{m_i}    \int_0^t e_i(s) \Gamma_{m^-(i)}(X(s))    ds
\end{align}
\begin{remark}
Note that  depending of every SDE one is interested in approximating using the WCE, the equations (\ref{xm-equation}) will be in fact a system of coupled ordinary differential equations. Such a system is called the propagator and can be solved iteratively (see \cite{lototsky2017stochastic} for instance).
\end{remark}

Some examples are:

\begin{itemize}
\item $F(x)= \alpha_0 x$ and $\Gamma(x)=\alpha_1  +\alpha_2 x$, then, if $\alpha_1=0$, $\alpha_2>0$ then we get the geometric Brownian motion, and when $\alpha_1>0$,$\alpha_2=0$ we recover the OU process.
\item $F(x)=  x (1-x)$ and   $\Gamma(x)=\alpha_1 x$, we obtain a multiplicative (affine) logistic SDE.
\end{itemize}

\subsubsection{Inverse problem for the SDE}

Let us to write  $X(t;\Theta)$ instead of $X(t)$, just to denote the dependence on the parameters, recall that $\Theta=(\theta_1,\theta_2)$.

We define the objective function, $u:\IR^2 \rightarrow \IR$, as 

\begin{equation}\label{obj_function}
u(\Theta):= \E\Bigg( \frac{1}{N} \sum_{k=1}^N \int_0^T [X^k(t;\Theta)-\hat{x}^k(t)]^2  dt\Bigg),
\end{equation}

where $\Theta=(\theta_1,\theta_2)$ are the parameters to be estimated and $X^k(t;\Theta)$ is the SDE with initial condition $X^k(0;\Theta)=\hat{x}^k(0)$. $\hat{x}^k(t)$ is the $k$ observed trajectory along time $[0,T]$, and we assume that we have  $N$ trajectories, and we will use them to estimate the parameters $\Theta$.

We now use the spectral representation (\ref{chaos-X}) in (\ref{obj_function}).
\begin{align}\label{u_1}
u(\Theta)&= \E\Bigg( \frac{1}{N} \sum_{k=1}^N \int_0^T \Big[ \sum_{ m\in\mathcal{J} }  X_m^k(t;\Theta) \xi_m -\hat{x}^k(t)\Big]^2 dt \Bigg)= \frac{1}{N} \sum_{k=1}^N \int_0^T \E\Big( \Big[ \sum_{ m\in\mathcal{J} }  X_m^k(t;\Theta) \xi_m -\hat{x}^k(t)\Big]^2  \Big)  dt \nonumber\\
&=  \frac{1}{N} \sum_{k=1}^N \int_0^T \E\Bigg( \Big[  X_{\mathbf{0}}^k(t;\Theta)  -\hat{x}^k(t)\Big] + \Big[ \sum_{\{m\in\mathcal{J}: |m|\ge 1\} }  X_m^k(t;\Theta) \xi_m \Big] \Bigg)^2  dt \nonumber\\
&= \frac{1}{N} \sum_{k=1}^N \int_0^T  \Big[  X_{\mathbf{0}}^k(t;\Theta)  -\hat{x}^k(t)\Big]^2  dt  + \frac{1}{N} \sum_{l=1}^N \int_0^T   \sum_{\{m\in\mathcal{J}: |m|\ge 1\} }  \Big[X_m^k(t;\Theta)  \Big]^2 dt.
\end{align}
where in the last step we have used the fact that $\{\xi_m\}$ is a basis for the space $L^2(\Omega)$. 

\begin{remark}
 Observe that from (\ref{u_1}) we can reinterpret the function $u(\Theta)$ as 
\begin{align}
  u(\Theta) &= \mbox{Square errors between the observations and the mean of $X$ } \nonumber \\
&\,\quad + \mbox{ Variance of the process $X$}.
\end{align}
\end{remark}

\begin{remark}
 We rewrite equation (\ref{xm-equation}) as follows, to underline the dependence on the parameters $\Theta=(\theta_1,\theta_2)$,  that is
\begin{align}\label{xm-equation1}
X_m(t;\Theta) &= x_0\I_{\{|m|= 0\} }  + g_1(\theta_1) \int_0^t  F_m(X(s;\Theta))ds + g_2(\theta_2) \sum_{i=1}^\infty \sqrt{m_i}    \int_0^t e_i(s) \Gamma_{m^-(i)}(X(s;\Theta))    ds
\end{align}
\end{remark}

\begin{remark}
It is worth noting that by employing the WCE, the original stochastic minimization problem is effectively reformulated as a deterministic optimization task. Consequently, the objective function is now evaluated through the evolution of the deterministic propagator equations. 
\end{remark}

\subsection{Stochastic partial differential equations}\label{SPDE-WCE}

We now present the case when the operator $\mathcal{A}$ corresponds to a partial differential equation. We follow  \cite{lototsky2006wiener} (see also \cite{lototsky2006stochastic}, or the book \cite{lototsky2017stochastic} for a more general type of SPDEs). We will use the notation and the version of the Wiener-chaos from the appendix.\\
Let $(V , H, V' )$ be a normal triple of Hilbert spaces so that $V \hookrightarrow  H \hookrightarrow  V'$ with both embeddings continuous.

We focus on the following stochastic equation
\begin{align}\label{stoc_PDE}
dX(t)&= A X(t)\, dt  \, + \, \sigma \sum_{k\ge 1} G_k(X(t)) \, dB_k(t),
\end{align}
subject to a suitable boundary and initial conditions. Here, the operator 
$A$ is a linear operator from   $V$ to $V'$. $\mathcal{H}$ is a Hilbert space of the form $L^2(\mathcal{O})$ with $\mathcal{O}$ being a bounded and smooth enough domain in $\R^d$. In addition, $\{B_k(\cdot)\}$ is a sequence of independent standard Wiener process, defined on a stochastic basis  $(\Omega,\mathcal{F},\{\mathcal{F}_t\}_{t\geq 0}, \Pb)$ that satisfies the usual assumptions.

We know that a solution of SPDE (\ref{stoc_PDE}) is given as follows. An $\mathcal{F}_t$-adapted process $X \in L^2(W; L^2((0,T); V))$ is
called a {\it square integrable} solution of (\ref{stoc_PDE})  if, for every $v \in V$ , there exists a measurable subset $\Omega'$ of $\Omega$ with $P(\Omega' ) = 1$, so that, for all $0 \le t \le T$ , the equality
$$
\big< X(t), v \big>_{\mathcal{H}} = \big< X(0), v \big>_{\mathcal{H}}+
\int_0^t \big< A X(s), v \big>_{\mathcal{H}} ds
+ \sum_{k\ge 1} \int_0^t  \, \big< G_k(X(s)), v \big>_{\mathcal{H}} \, dB_k(s)
$$
holds on $\Omega'$.

Furthermore, in \cite{lototsky2006wiener}  the author present the following result.
An  $\mathcal{F}_t$-adapted process $X$ is a square integrable solution of (\ref{stoc_PDE})  is and only if 
$$
X(t) = \sum_{\alpha \in \mathcal{J} } X_\alpha\, \xi_\alpha
$$
so that the Fourier-Hermite coefficients $X_\alpha$ satisfy a suitable integrability condition and solves a system of coupled PDEs, so-called the propagator (see Th. 3.4 in \cite{lototsky2006wiener}). Indeed, $X_\alpha$ solves
\begin{align}\label{propagator_SPDE}
 X_\alpha(t)&= X(0)\I_{|\alpha|=0} +  \int_0^t A X_\alpha(s) \, ds + \sum_{i,k} \sqrt{\alpha_i} \int_0^t  \, G_k(X_{\alpha^-(i)}(s)) \, e_i(s)\, ds.
\end{align}

Thus, we could write the solution of the SPDE (\ref{stoc_PDE}) using the WCE as

\begin{equation}\label{chaos-SPDE}
X(t)=\sum_{\alpha \in\mathcal{J}} X_\alpha(t)\, \xi_m 
\end{equation}
where $X_\alpha(t)$ solves the system (\ref{propagator_SPDE}) for $\alpha \in\mathcal{J}$.

\subsubsection{Inverse problem for the SPDE}
We write by $X(t;\Theta)=X(t)$ to stand out the dependence of the SPDE of some parameters that one could be interested. For instance, if we consider a stochastic heat equation, we could be interested in estimate the diffusivity parameter, besides the diffusion parameter.

As before, we define the objective function, $u:\IR^2 \rightarrow \IR$, as 

\begin{equation}\label{obj_function2}
u(\Theta):= \E\Bigg( \frac{1}{N} \sum_{k=1}^N \int_0^T [X^k(t;\Theta)-\hat{x}^k(t)]^2  dt\Bigg),
\end{equation}

where $\hat{x}^k(t)$ is the $k$ realization of the SPDE along time $[0,T]$, and we assume that we have  $N$ of them.

We  use the spectral representation (\ref{chaos-SPDE}) in (\ref{obj_function2}) and, as before, we obtain
\begin{eqnarray}\label{u_2}
u(\Theta)&=&  \frac{1}{N} \sum_{k=1}^N \int_0^T  \Big[  X_{\mathbf{0}}^k(t;\Theta)  -\hat{x}^l(t)\Big]^2  dt  + \frac{1}{N} \sum_{l=1}^N \int_0^T   \sum_{\{m\in\mathcal{J}: |m|\ge 1\} }  \Big[X_m^k(t;\Theta)  \Big]^2 dt.
\end{eqnarray}
which is a well-posed problem since we recognize the Tikhonov–Phillips regularized solution for some deterministic problem.

\section{Learning parameters for SDEs}\label{sec:LearSDEs}
\subsection{Stochastic differential equations}
Suppose that a process $X(t)$  is given by the stochastic differential equation (SDE)
 \begin{align}
 dX(t)&= f(X(t))dt + \sigma(X(t))dB_t, \\
X(0)&=x_0.\nonumber
 \end{align}

 We assume that the functions $f$ and $\sigma$ satisfy suitable assumptions to ensure the existence and uniqueness of a solution (see \cite{oksendal2013stochastic} for instance). In particular, we will assume that $X\in L^2(\Omega; [0,T])$.
In addition, since we are looking to estimate some parameters from the SDE  \eqref{SDE1} then we will make explicit the dependence of the functions $f$  and $\sigma$ on some parameters. More precisely, we suppose that $f(x)=f(x;\theta_1)$  and $\sigma(x)=\sigma(x,\theta_2)$. Here,  $(\theta_1,\theta_2)$, are the parameters we are interested in.

 Using the Hypotheses \ref{hypo1} this, we have that
\begin{equation}\label{partial0}
 \frac{\partial f(x;\Theta)}{\partial \theta_2} = \frac{\partial \sigma(x,\Theta)}{\partial \theta_1} = 0.
\end{equation}
  
 Note that since we are interested in estimating the parameters $\theta_1$ and $\theta_2$, thus we assume that the functions $g_1(\theta_1)$ and $g_2(\theta_2)$ are independent of time and deterministic.

\subsection{Wiener Chaos expansion for the SDE} 

It is recalled that the SDE \eqref{SDE1} admits the following representation:
\begin{equation}
X(t)=\sum_{m\in\mathcal{J}} X_m(t) \xi_m 
\end{equation}

where $ X_m(t)=\E\big( X(t) \xi_m \big) $, solves

\begin{align}
X_m(t) &= x_0\I_{\{|m|= 0\} }  +   g_1(\theta_1) \int_0^t  F_m(X(s))ds +  g_2(\theta_2) \sum_{i=1}^\infty \sqrt{m_i}    \int_0^t e_i(s) \Gamma_{m^-(i)}(X(s))    ds
\end{align}

\subsection{Learning parameters using the WCE} From this point we write  $X(t;\Theta)$ instead of $X(t)$, to denote the dependence on the parameters, recall that $\Theta=(\theta_1,\theta_2)$.

We are interested in estimating the parameters $\Theta=(\theta_1,\theta_2)$ using the stochastic gradient descent (SGD) method (see \cite{buduma2022fundamentals} or \cite{ketkar2017stochastic} for instance)

The SGD method usually estimates the parameters by minimizing an objective function. Suppose that we have observed $N$ trajectories of the process $X$ in the interval $[0,T]$ and denote the $k$-th trajectory by $\hat{x}^k(t)$. 

We will assume that, for $T>0$, we have observed $\hat{x}^k(t)$ in the interval  $[0, T]$ in times
$0=t_0<t_1\ldots,t_n=T$ with $n$ large enough such that, for $i=1,\ldots,n$, $\Delta_n:=t_i-t_{i-1}=\tfrac{T}{n}$  is close to zero. 

Define the quadratic variation of $\hat{x}^k(t)$, up to time $t$, as

$$
\langle\hat{x}^k,\hat{x}^k\rangle_t:=\sum_{i=1}^{n^*}\, [x^k(t_i)-x^k(t_{i-1})]^2.
$$
 where $n^*$ is such that $t_{n^*} \le t< t_{n^*+1} $.

 In a similar vein, we also define the process
 
 $$
\langle\hat{x}^k\rangle_t^2:= \frac{\Delta_n}{2} \sum_{i=1}^{n^*}\,\Big( [x^k(t_i)]^2+[x^k(t_{i-1})]^2 \Big) \approx \int_0^t [x^k(t)]^2 dt.
$$

We recall the {\it objective function}, also called the {\it loss function}, $\widehat{u}:\IR^p \rightarrow \IR$, as 

\begin{equation}
\widehat{u}(\Theta):= \E\Bigg( \int_0^T \Big[X(t;\Theta)-\frac{1}{N} \sum_{k=1}^N \hat{x}^k(t) \Big]^2  dt\Bigg),
\end{equation}

where $\Theta=(\theta_1,\ldots,\theta_p)$ are the parameters to be estimated and $X(t;\Theta)$ is the SDE with initial condition $X(0;\Theta)=\hat{x}(0)$ which is the initial condition for all the observed trajectories.

By recalling \eqref{u_1} we have
\begin{align}
\widehat{u}(\Theta)&=   \int_0^T  \Big[  X_{\mathbf{0}}(t;\Theta)  -  \frac{1}{N} \sum_{k=1}^N\hat{x}^k(t)\Big]^2  dt +  \int_0^T   \sum_{\{m\in\mathcal{J}: |m|\ge 1\} }  \Big[X_m(t;\Theta)  \Big]^2 dt.
\end{align}

However, for numerical purposes the expression \eqref{u_2} is not useful to estimate the diffusion parameter; thus we propose the following modification of the functional $\widehat{u}$. We define

\begin{align}\label{u_1_1}
u(\Theta)&=  \Bigg(  \int_0^T  \Big[  X_{\mathbf{0}}(t;\Theta)  - \frac{1}{N}\sum_{k=1}^N \hat{x}^k(t)\Big]^2  dt \Bigg)  + \, \Bigg(  \int_0^T \Bigg[  \sum_{\{m\in\mathcal{J}: |m|\ge 1\} }  \Big[X_m(t;\Theta)  \Big]^2  - \frac{1}{N}\sum_{k=1}^N \frac{  \langle\hat{x}^k,\hat{x}^k\rangle_t }{\langle\hat{x}^k\rangle_t^2  } \Bigg]^2 dt \Bigg)
\end{align}

 This paper will address the case of $p=2$, which means we will estimate two parameters, however, transitioning to the general case is a straightforward process.

\begin{remark}
It's important to highlight that the recent modification in the last term, ensures the model's energy fluctuations align with the empirical quadratic variation observed in the data.
\end{remark}

\begin{remark}
To underline the dependence on the parameters $\Theta=(\theta_1,\theta_2)$,  we rewrite the Equation \eqref{xm-equation} as follows,

\begin{align}\label{xm-equation2}
X_m(t;\Theta) &= x_0\I_{\{|m|= 0\} }  + g_1(\theta_1) \int_0^t  F_m(X(s;\Theta))ds + g_2(\theta_2) \sum_{i=1}^\infty \sqrt{m_i}    \int_0^t e_i(s) \Gamma_{m^-(i)}(X(s;\Theta))    ds
\end{align}

\end{remark}

\begin{remark}
Observe that by using the WCE we have now the problem of minimizing the function $u$ is now equivalent to minimizing a {\it deterministic} set of equations \eqref{u_1_1}, which is derived from the propagator. 
\end{remark}

At this point we calculate the gradient of the function $u:\IR^2 \rightarrow \IR$ which is defined at the point  $\Theta=(\theta_{1},\theta_{2})$ as 
$$
V(\Theta)= \nabla u(\Theta) = {\begin{bmatrix}{\frac {\partial u}{\partial \theta_{1}}}(\Theta)\\{\frac {\partial u}{\partial \theta_{2}}}(\Theta)\end{bmatrix}}.
$$

Thus, we write the partial derivative from the last expression and we have that for $j=1,2$ they are

\begin{align}\label{partial_uj_1}
\frac{\partial u}{\partial \theta_j}(\Theta) &= 2 \Bigg( \int_0^T  \Big[  X_{\mathbf{0}}(t;\Theta)  - \frac{1}{N} \sum_{k=1}^N \hat{x}^k(t)\Big]    \,\cdot \, \frac{\partial  X_{\mathbf{0}}(t;\Theta) }{\partial \theta_j} dt \Bigg)\nonumber\\ 
& + 4  \Bigg(  \int_0^T \Bigg[  \sum_{\{m\in\mathcal{J}: |m|\ge 1\} }  \Big[X_m (t;\Theta)  \Big]^2  - \frac{1}{N}\sum_{k=1}^N \frac{  \langle\hat{x}^k,\hat{x}^k\rangle_t }{\langle\hat{x}^k\rangle_t^2  } \Bigg] 
\cdot  \sum_{\{m\in\mathcal{J}: |m|\ge 1\} } X_{\mathbf{m}}(t;\Theta) \,\frac{\partial  X_{\mathbf{m}}(t;\Theta) }{\partial \theta_j}\, dt \Bigg) 
\end{align}

For any $m\in \mathcal{J}$ we calculate $\frac{\partial  X_m^k(t;\Theta) }{\partial \theta_j} $ from \eqref{xm-equation2}.   At this point we use the Hypothesis \ref{hypo1}; more precisely, we use \eqref{partial0}, this let us

\begin{align}\label{xm-derivative}
\frac{\partial  X_m^k(t;\Theta) }{\partial \theta_1} &=  g_1'(\theta_1) \int_0^t    X^k(s;\Theta)) ds \,+ \, g_1(\theta_1) \int_0^t    \frac{\partial F_m(X^k(s;\Theta))} {\partial \theta_1} ds \\
\frac{\partial  X_m^k(t;\Theta) }{\partial \theta_2} &=  g_2'(\theta_2) \sum_{i=1}^\infty \sqrt{m_i}    \int_0^t e_i(s)  \Gamma_{m^-(i)}(X^k(s;\Theta))    ds  + g_2(\theta_2) \sum_{i=1}^\infty \sqrt{m_i}    \int_0^t e_i(s) \frac{\partial \Gamma_{m^-(i)}(X^k(s;\Theta))  } {\partial \theta_2}    ds
\end{align}

Then, using \eqref{partial0} we obtain the partial derivatives of u:
\begin{align}\label{partial_u1}
\frac{\partial u}{\partial \theta_1}(\Theta) &=2 g_1'(\theta_1) \int_0^T  \Big[  X_{\mathbf{0}}(t;\Theta)  -\frac{1}{N} \sum_{k=1}^N \hat{x}^k(t)\Big] \cdot    \Bigg(\int_0^t     F_\mathbf{0}(X^k(s;\Theta)) ds\Bigg) dt\nonumber\\ 
 &\quad +2 g_1(\theta_1) \int_0^T  \Big[  X_{\mathbf{0}}(t;\Theta)  -\frac{1}{N}\sum_{k=1}^N \hat{x}^k(t)\Big] \cdot    \Bigg(\int_0^t    \frac{\partial F_\mathbf{0}(X^k(s;\Theta))} {\partial \theta_1} ds\Bigg) dt\nonumber\\ 
&\quad +2 g_1'(\theta_1)
\sum_{k=1}^N \int_0^T  \Bigg[  \sum_{\{m\in\mathcal{J}: |m|\ge 1\} }  \Big[X_m (t;\Theta)  \Big]^2  - \frac{1}{N}\sum_{k=1}^N \frac{  \langle\hat{x}^k,\hat{x}^k\rangle_t }{\langle\hat{x}^k\rangle_t^2  } \Bigg] \cdot  \Bigg(\int_0^t     F_m(X^k(s;\Theta)) ds\Bigg) dt\nonumber\\ 
&\quad +2 g_1(\theta_1) \sum_{k=1}^N \int_0^T   \Bigg[  \sum_{\{m\in\mathcal{J}: |m|\ge 1\} }  \Big[X_m (t;\Theta)  \Big]^2  - \frac{1}{N}\sum_{k=1}^N \frac{  \langle\hat{x}^k,\hat{x}^k\rangle_t }{\langle\hat{x}^k\rangle_t^2  } \Bigg] \cdot  \Bigg(\int_0^t    \frac{\partial F_m(X^k(s;\Theta))} {\partial \theta_1} ds\Bigg) dt,
\end{align}
and 
\begin{align}\label{partial_u2}
\frac{\partial u}{\partial \theta_2}(\Theta) &= g_2'(\theta_2) \frac{2}{N} \sum_{k=1}^N \int_0^T  \Big[  X_{\mathbf{0}}^k(t;\Theta)  -\hat{x}^k(t)\Big] \cdot \Bigg(  \sum_{i=1}^\infty \sqrt{0}    \int_0^t e_i(s)  \Gamma_{\mathbf{0}^-(i)}(X^k(s;\Theta))     ds \Bigg)   dt\nonumber\\ 
&\quad+ g_2(\theta_2) \frac{2}{N} \sum_{k=1}^N \int_0^T  \Big[  X_{\mathbf{0}}^k(t;\Theta)  -\hat{x}^k(t)\Big] \cdot \Bigg(  \sum_{i=1}^\infty \sqrt{0}    \int_0^t e_i(s) \frac{\partial \Gamma_{\mathbf{0}^-(i)}(X^k(s;\Theta))  } {\partial \theta_2}    ds \Bigg)   dt\nonumber\\ 
&\quad +  g_2'(\theta_2) \frac{2}{N} \sum_{k=1}^N \int_0^T   \sum_{\{m\in\mathcal{J}: |m|\ge 1\} }  X_m^k(t;\Theta)  \cdot \Bigg(  \sum_{i=1}^\infty \sqrt{m_i}    \int_0^t e_i(s)  \Gamma_{m^-(i)}(X^k(s;\Theta))     ds \Bigg)  dt\nonumber\\ 
&\quad +  g_2(\theta_2) \frac{2}{N} \sum_{k=1}^N \int_0^T   \sum_{\{m\in\mathcal{J}: |m|\ge 1\} }  X_m^k(t;\Theta)  \cdot \Bigg(  \sum_{i=1}^\infty \sqrt{m_i}    \int_0^t e_i(s) \frac{\partial \Gamma_{m^-(i)}(X^k(s;\Theta))  } {\partial \theta_2}    ds \Bigg)  dt\nonumber\\ 
& =   g_2'(\theta_2) \frac{2}{N} \sum_{k=1}^N \int_0^T   \sum_{\{m\in\mathcal{J}: |m|\ge 1\} }  X_m^k(t;\Theta)  \cdot \Bigg(  \sum_{i=1}^\infty \sqrt{m_i}    \int_0^t e_i(s)  \Gamma_{m^-(i)}(X^k(s;\Theta))     ds \Bigg)  dt\nonumber\\ 
&\quad +  g_2(\theta_2) \frac{2}{N} \sum_{k=1}^N \int_0^T   \sum_{\{m\in\mathcal{J}: |m|\ge 1\} }  X_m^k(t;\Theta)  \cdot \Bigg(  \sum_{i=1}^\infty \sqrt{m_i}    \int_0^t e_i(s) \frac{\partial \Gamma_{m^-(i)}(X^k(s;\Theta))  } {\partial \theta_2}    ds \Bigg)  dt.\\
\end{align}
 It is well-known that $u(\Theta )$ decreases fastest if one goes from, an initial, $\Theta\in\IR^2$, in the direction of the negative Gradient of  $u$  at $\Theta  \, -\nabla u(\Theta )$ (see for instance \cite{chong2004introduction}).
 Thus, it follows that it is possible to start with a guess $\Theta_0\in\IR^2$ and consider the sequence 
 $$
 \Theta_{k+1} = \Theta_k \, - \gamma \nabla u(\Theta_k ),
 $$
 for a small enough step size or {\it learning rate} $\gamma \in \IR^{+}$. Then, 
 $u(\Theta_k ) \ge u(\Theta_{k+1} )$

We will consider $\gamma$ as in the Barzilai–Borwein method (see \cite{fletcher2005barzilai}
):

$$
\gamma_{k}={\frac {\left|\left(\Theta_{k}-\Theta_{k-1}\right)^{T}\left[\nabla u(\Theta_{k})-\nabla u(\Theta_{k-1})\right]\right|}{\left\|\nabla  u(\Theta_{k})-\nabla u(\Theta_{k-1})\right\|^{2}}}
$$

With some additional hypotheses and this choice of $\gamma$, we have that the convergence to a local minimum can be guaranteed. (see \cite{chong2004introduction},  
\cite{Tan-Co2016}, \cite{garrigos2024}). 

\subsubsection{Convergence of Stochastic Gradient}

To ensure convergence using the Barzilai–Borwein method, certain hypotheses must be satisfied. Let the objective function be defined as the empirical risk $u(\Theta)=\frac{1}{N} \sum_{i=1}^N u_i (\Theta)$. For the SGD-BB algorithm to converge, each component $u_i$ must satisfy the following properties:
\begin{enumerate}
    \item The function $u_i$ is $\mu$-strongly convex, i.e., there exists $\mu > 0$ such that for all $\Theta, \Lambda$: 
    $$u_i (\Theta)\geq u_i (\Lambda) +\nabla u_i (\Lambda)^T (\Theta-\Lambda)+\frac{\mu}{2}\|\Theta-\Lambda\|_2^2.$$
    \item The gradient of $u_i$ is $L$-Lipschitz continuous, i.e., there exists $L > 0$ such that: 
    $$\|\nabla u_i (\Theta)-\nabla u_i(\Lambda)\|_2\leq L \|\Theta-\Lambda\|_2 .$$
\end{enumerate}

First, observe that the core component of the objective function, given by the $L^2$ error of the mean trajectory:
$$
\int_0^T \Big[ X_{\mathbf{0}}(t;\Theta) - \frac{1}{N}\sum_{k=1}^N \hat{x}^k(t)\Big]^2 dt,
$$
is convex, a property that follows from the separation of parameters established in Hypotheses \ref{hypo1}. Furthermore, since the full objective functional $u(\Theta)$ incorporates Tikhonov regularization, the problem becomes strictly strongly convex, as proven in Theorem \ref{tikhonov} (see also \cite[Lemma~2.12]{garrigos2024}). \\
To verify the Lipschitz condition for the gradient, we introduce the notation $\bar{\hat{x}}(t)=\frac{1}{N} \sum_{k=1}^N \hat{x}^k(t)$. Evaluating the difference of the gradients for two distinct parameter vectors $\Theta_1$ and $\Theta_2$ yields:
\begin{align}
   & \int_0^T  \Big[  X_{\mathbf{0}}(t;\Theta_1)  - \bar{\hat{x}}(t)\Big]    \frac{\partial  X_{\mathbf{0}}(t;\Theta_1) }{\partial \theta_j} dt -\int_0^T  \Big[  X_{\mathbf{0}}(t;\Theta_2)  - \bar{\hat{x}}(t)\Big]    \frac{\partial  X_{\mathbf{0}}(t;\Theta_2) }{\partial \theta_j} dt \nonumber \\
    &= \int_0^T \Bigg( \Big[ X_{\mathbf{0}}(t;\Theta_1) - \bar{\hat{x}}(t) \Big] \Big[\frac{\partial X_{\mathbf{0}}(t;\Theta_1) }{\partial \theta_j} - \frac{\partial X_{\mathbf{0}}(t;\Theta_2) }{\partial \theta_j}\Big]  + \Big[ X_{\mathbf{0}}(t;\Theta_1) - X_{\mathbf{0}}(t;\Theta_2) \Big] \frac{\partial X_{\mathbf{0}}(t;\Theta_2) }{\partial \theta_j} \Bigg) dt. \label{eq:lipschitz_bound}
\end{align}
From Equation \eqref{eq:lipschitz_bound}, it becomes evident that we must impose the hypothesis that $X_{\mathbf{0}}$ has continuous second-order derivatives with respect to $\Theta$, and that the parameter space $\Theta$ is constrained to a compact, bounded interval. Under these conditions, $X_{\mathbf{0}}$ and its derivatives are bounded. Consequently, both the difference of the states $\big|X_{\mathbf{0}}(\Theta_1) - X_{\mathbf{0}}(\Theta_2)\big|$ and the difference of their derivatives can be strictly bounded by a constant multiple of $\|\Theta_1 - \Theta_2\|_2$. 
By applying the triangle inequality and the Cauchy-Schwarz inequality to the integral, the gradient of the objective function satisfies the $L$-Lipschitz condition. Combining this with the strong convexity induced by the regularization, we guarantee convergence to a local minimum, as established in \cite[Theorem~3.8]{Tan-Co2016}.

Finally, the strong convexity of the regularized objective function $u$ implies that any local minimum is uniquely the global minimum. Therefore, the sequence generated by the stochastic gradient descent is mathematically guaranteed to converge to the global optimum of the parameter space.

\begin{algorithm}
    \caption{Gradient Descent}\label{alg:gd}
    \begin{algorithmic}[1]
        \REQUIRE $\{x^k\}_{k=1}^N $ trajectories of the Process observed in the interval  $[0, T]$ in times $0=t_0<t_1\ldots,t_n=T$ and $\tilde\theta_{10}$ as first estimation of $\theta_1$
        \STATE Calculate $\tilde\sigma=\frac{1}{N}\sum_{k=1}^N \frac{  \langle\hat{x}^k,\hat{x}^k\rangle_t }{\langle\hat{x}^k\rangle_t^2  }$
        \STATE Initialize $\Theta_{2} = \Theta_1 \, -  \nabla u(\Theta_1 ),$ where $\Theta_1 =\begin{bmatrix}\tilde\theta_{10} \\ \tilde\sigma\end{bmatrix}$.
        \WHILE{$m\leq m_{\max}$}
            \STATE Calculate $\Theta_{k+1} = \Theta_k \, - \gamma_{k} \nabla u(\Theta_k ),$ where $\gamma_{k}={\frac {\left|\left(\Theta_{k}-\Theta_{k-1}\right)^{T}\left[\nabla u(\Theta_{k})-\nabla u(\Theta_{k-1})\right]\right|}{\left\|\nabla  u(\Theta_{k})-\nabla u(\Theta_{k-1})\right\|^{2}}}$
        \ENDWHILE
        \RETURN $\Theta_{m_{\max}}$
    \end{algorithmic}
\end{algorithm}

\begin{algorithm}
    \caption{Gradient Descent for OU process}\label{alg:gdou}
    \begin{algorithmic}[1]
        \REQUIRE $\{x^k\}_{k=1}^N $ trajectories of the Process observed in the interval  $[0, T]$ in times $0=t_0<t_1\ldots,t_n=T$ and $\tilde\theta_{10}$ as first estimation of $\theta_1$
        \STATE Calculate $\tilde\sigma=\frac{1}{N}\sum_{k=1}^N \frac{  \langle\hat{x}^k,\hat{x}^k\rangle_t }{\langle\hat{x}^k\rangle_t^2  }$
        \STATE Initialize $\Theta_{2} = \Theta_1 \, -  \nabla u(\Theta_1 ),$ where $\Theta_1 =\begin{bmatrix}\tilde\theta_{10} \end{bmatrix}$.
        \WHILE{$m\leq m_{\max}$}
            \STATE Calculate $\Theta_{k+1} = \Theta_k \, - \gamma_{k} \nabla u(\Theta_k ),$ where $\gamma_{k}={\frac {\left|\left(\Theta_{k}-\Theta_{k-1}\right)^{T}\left[\nabla u(\Theta_{k})-\nabla u(\Theta_{k-1})\right]\right|}{\left\|\nabla  u(\Theta_{k})-\nabla u(\Theta_{k-1})\right\|^{2}}}$
        \ENDWHILE
        \RETURN $\Theta_{m_{\max}}$
    \end{algorithmic}
\end{algorithm}

 \section{Numerical experiments}\label{sec:NumExp}

In this section, we provide a numerical simulation of SDEs  using its Wiener-chaos expansion as was described in previous sections. Indeed, solving the system given by Equations \eqref{xm-equation} , and \eqref{xm-equation2}, we fixed the coefficients for the Wiener-chaos expansion of the stochastic processes $X$.

An algorithm to solve the system given in \eqref{xm-equation2} is presented in the appendix.

\subsection{A toy example}

Set the Hilbert space $\mathcal{H}=L^2(0,T)$ and
consider the Ornstein-Uhlenbeck (OU) process given by the SDE 
\begin{align}\label{OU-SDE}
dx(t) &= -a x(t) dt + \sigma dB(t),\quad\mbox{with initial condition }
x(0)= x^0, 
\end{align}
We have its integral representation 
\begin{align}\label{OU-integral}
x(t) &= x^0 -a \int_0^t x(s) ds + \sigma \int_0^t  dB(s). 
\end{align}

Observe that in this case $\theta_1=a$ and $\theta_2=\sigma$. Furthermore, $$
f(x:\theta_1)= -\theta_1 x \qquad\mbox{and } \qquad \sigma(x;\theta_2)=\theta_2
$$

With suitable assumptions we have that $x\in L^2(\Omega, \mathcal{F},\Pb)$, this implies that we can write
$$
x(t)= \sum_{\alpha \in \mathcal{J}} X_\alpha(t) \xi_\alpha, \qquad \mbox{with} \, X_\alpha(t)= \E(x(t)\, \xi_\alpha),
$$
thus, using \eqref{def-delta} we get
\begin{align}\label{propagator_OU}
    X_\alpha(t)
        &= \E\Big[\Big(x^0  -a \int_0^t x(s) ds 
            + \sigma \int_0^t  dB(s) \Big)\xi_\alpha\Big] \nonumber\\   
        &=  x^0\I_{|\alpha|=0}   -a \int_0^t  \E\Big(x(s) ds \xi_\alpha \Big) 
            + \sigma \E\Big[ \Big(\int_0^t  dB(s)\Big) \xi_\alpha\Big] \nonumber\\
        &= x^0\I_{|\alpha|=0}  - a \int_0^t  X_\alpha(s) ds 
            + \sigma \E\Big[ \big\langle \I_{[0,t]}(\cdot), D\xi_{\alpha}
            \big\rangle_{\mathcal{H}} \Big] \nonumber\\
        &= x^0\I_{|\alpha|=0}  - a \int_0^t  X_\alpha(s) ds 
            + \sigma \int_0^t \sum_{i,j=1}^\infty \sqrt{\alpha_{i,j}} \,\, 
                \E \big( \xi_{\alpha^-(i,j)} \big) ds\nonumber\\
        &=x^0\I_{|\alpha|=0}  - a \int_0^t  X_\alpha(s) ds 
            + \sigma t \sum_{i,j=1}^\infty \sqrt{\alpha_{i,j}}\,\, \I_{|\alpha^-(i,j)|=0}.
\end{align}

From this expression, we now calculate the propagator for the following three cases
$|\alpha|=0$, $|\alpha|=1$ and $|\alpha|>1$.

For $|\alpha|=0$ denote the propagator as $ X_0= X_{|\alpha|=0}$, observe that in
this case the indicator function $\I_{|\alpha^-(i,j)|=0}$ is equal to zero, then
from \eqref{propagator_OU} we have the ordinary differential equation

\begin{align*}
\frac{dX_0(t)}{dt}
&= - a   X_0(t),
\end{align*}
with initial condition $x^0$. This means that the solution $X_0(t)$ is

\begin{align}\label{m0_OU_prev}
 X_0(t)
&= x^0 \exp(- at).
\end{align}

For the case $|\alpha|=1$ denote the propagator as   $ X_1= X_{|\alpha|=1}$. 

Then, from \eqref{propagator_OU} and given that $|\alpha|=1$ then we have  $\sqrt{\alpha_{i}}=1 $ for some $i$, thus we can write for such $i$

\begin{align}\label{m1_OUdiff}
    \frac{dX_1(t)}{dt}
        &= - a X_1(t)  + \sigma  \sum_{i=1}^\infty \sqrt{\alpha_{i}}\,\,
        e_i(t)  \I_{\{i=1\} } \,\, \I_{\{|\alpha^-(i)|=0\}}\nonumber\\
        &= - a   X_1(t) + \sigma e_i(t) 
\end{align}
where  the ODE has zero initial condition.
Thus, we have that 

\begin{align}\label{m1_OU}
 X_{1_i}(t) =  \sigma  e^{-at} \int_0^t e^{as} e_i(s) ds.
\end{align}

For the case $|\alpha|>1$ we have the ODE with zero initial condition 

\begin{align}\label{m1_OU_alphagrather1}
\frac{dX_\alpha(t)}{dt}
&= - a   X_\alpha(t), 
\end{align}
with the trivial solution $X_\alpha(t) \equiv 0$.

This implies that the WCE for the OU process could be written as
\begin{align}\label{OU_WCE_approx}
x(t) = \sum_{\alpha \in \mathcal{J}} X_\alpha(t) \xi_\alpha = X_0(t) + \sum_{k =1}^\infty  X_{1_k}(t) \xi_{\alpha_k}\quad \mbox{with }\quad
 \xi_{\alpha_k} :=  \int_0^T e_k(s) dB(s)
 \end{align} 

\subsection{SGD for the OU process}
For the  OU process we will use the function 

\begin{align}\label{u_1_ou}
u(\Theta)&=  \Bigg(  \int_0^T  \Big[  X_{\mathbf{0}}(t;\Theta)  - \frac{1}{N}\sum_{k=1}^N \hat{x}^k(t)\Big]^2  dt \Bigg) 
\end{align}
to estimate the drift parameter.

\subsubsection{OU Process Results}

\begin{center}
\begin{tabular}{ |c|c|c| } 
\hline
 Parameters & Initial condition & Final Estimation \\
\hline
$\alpha=1.7 ,\sigma=0.15$ & $\tilde\alpha_0=2 ,\tilde\sigma=0.1331$ & $\tilde\alpha_n= 1.71562$ \\ 
 & $\tilde\alpha_0=0.5,\tilde\sigma=0.1331$ & $\tilde\alpha_n=1.66387$ \\ 
 & $\tilde\alpha_0=1.01,\tilde\sigma=0.1331$ & $\tilde\alpha_n=1.71266$ \\
\hline
 $\alpha=0.29 ,\sigma=0.1$ & $\tilde\alpha_0=2 ,\tilde\sigma=0.08801$ & $\tilde\alpha_n= 0.289461$ \\ 
 & $\tilde\alpha_0=0.5,\tilde\sigma=0.08801$ & $\tilde\alpha_n=0.289451$ \\ 
 & $\tilde\alpha_0=0.37,\tilde\sigma=0.08801$ & $\tilde\alpha_n=0.289462$ \\
\hline
$\alpha=1.26 ,\sigma=0.076$ & $\tilde\alpha_0=1.55 ,\tilde\sigma=0.06833$ & $\tilde\alpha_n= 1.23858$ \\ 
 & $\tilde\alpha_0=0.8,\tilde\sigma=0.06833$ & $\tilde\alpha_n=1.23354$ \\ 
 & $\tilde\alpha_0=0.45,\tilde\sigma=0.06833$ & $\tilde\alpha_n=1.2315$ \\
 \hline
\end{tabular}
\end{center}

\section{Applications to some SDEs}\label{sec:AppExamSDEs}

\subsection{Geometric Brownian motion}

Consider the SDE 
    \begin{equation}\label{GBME}
        dX_t = \alpha X_t dt + \sigma X_t dW_t.
    \end{equation}
with $X_0 = x_0$. The process solution of this equation is called the Geometric Brownian Motion (GBM) (see \cite{oksendal2013stochastic} for a review of this process). This process serves as the canonical model for continuous-time stochastic processes where the variable is constrained to be strictly positive, most notably in the modeling of asset prices within mathematical finance. It was famously adopted by \cite{Samuelson1965} and later underpinned the Black-Scholes option pricing framework \cite{Black1973}.

We have the unknown parameters $\alpha$ and $\sigma$, and we want to estimate them. We employ the Wiener chaos method utilizing a truncation of the WCE with $L=1,000$ and  $p=8$.

Following \cite{delgado2024simulating}  the propagator for this SDE is calculated as follows. For the case when $|m|=0$, we  denote by $X_0$ the solution of the propagator which solves the differential 
    equation
    \begin{align}\label{xm-equation0}
        X_0(t) 
            &= x_0 + a \int_0^t  X_0(s) ds, \quad \mbox{which is given by  } \, X_0(t)=x_0\exp(a t).
    \end{align}
    For the case $|m| \ge 1$ and using the vectors from Table 5 in \cite{delgado2024simulating}, $X_{m}^i$ satisfies
    \begin{align}
        X_{m}^i(t) 
            &= a \int_0^t  X_{m}^i(s) ds+ \sigma\sum_{j} 
            \sqrt{m_j^i} \int_0^t  e_j(s)\, X_{m^-(j)}^i(s)\, ds,
    \end{align}
    where $m^{i-}(j)$ is the $i$-th vector with the $j$-th entry minus $1$\footnote{see 
    the definition of  $m^{-}(j)$  in \eqref{m-i}}.
      We should numerically solve this system of coupled ODEs,  in increasing order of $|m|$.

\subsection{Stochastic logistic equation}\label{Sec:SLDE}
The classical logistic growth model, originally proposed by Verhulst, describes the dynamics of a population $N(t)$ constrained by a carrying capacity $K$. For any positive initial condition $P(0) > 0$, the population monotonically approaches the equilibrium state, $K$, ensuring persistence. However, biological systems are inevitably subject to environmental fluctuations—such as temperature variations or resource availability—that render the growth rate time-dependent and stochastic. To model this, the parameter $r$ is typically perturbed by a white noise term, $r \to r + \sigma \dot{B}(t)$, where $\sigma$ denotes the noise intensity and $\dot B(t)$ is Gaussian white noise. This transformation leads to the \textit{Stochastic Logistic Differential Equation} \eqref{sto-log} which is interpreted in the It\^o sense:

\begin{equation}\label{sto-log}
  \left.\begin{aligned} dP(t)&=   r\, P(t)\Big[1-P(t)\Big] dt + \sigma P(t) dB(t),\quad t>0\\
 P(0)&= p_0.
 \end{aligned}
 \right\}
 \end{equation}

where $r$ represents the intrinsic growth rate.  In addition, the diffusion term $\sigma P(t) dB(t)$ is multiplicative (linear in $P(t)$), reflecting that environmental noise has a greater absolute impact on larger populations while vanishing when the population is zero (an absorbing state). 
We seek to estimate the parameters $r$ and $\sigma$ using the proposed method.
\\
As before, we calculate the propagator following \cite{delgado2024simulating}.  We denote by $P_0$ the solution for the case $|m|=0$. Thus,  $P_0$ solves the differential equation
    \begin{align}\label{xm-log0}
        P_0(t) 
            &= p_0 + r \int_0^t  P_0(s)\big[1-P_0(s)\big] ds, 
    \end{align}
    which has as a solution
    \begin{align}\label{xm-log0-sol}
        P_0(t) 
            &=\frac{p_0}{1+\tfrac{1-p_0}{p_0} e^{-r t}}. 
    \end{align}
    
    We now consider the case where $|m| = 1$, and we will use the vectors from Table 5 in \cite{delgado2024simulating}. Here we will apply a truncation as follows. For the non-linear logistic drift, the exact Galerkin projection of the quadratic term $P_t^2$ onto the $j$-th chaos mode introduces the third-order structure tensor $c_{ikj} = \mathbb{E}[\xi_i \xi_k \xi_j]$. While $c_{ikj}$ vanishes exactly when the sum of the polynomial degrees is odd, it yields non-zero contributions for even-order combinations. However, the products of higher-order purely stochastic coefficients ($P_i P_k$) are typically negligible compared to the mean-field cross-terms ($2P_1 P_j$). To maintain computational tractability and avoid the curse of dimensionality associated with the 3D tensor, we employ a standard first-order truncation for the non-linear interactions, approximating the drift of $dP_j$ as $\approx \alpha P_j (1 - 2P_0)$. This linearized approximation effectively captures the macro-scale biological saturation while preserving the efficiency of the SGD algorithm.
 Thus $P_{m}^i$ satisfies
    
    \begin{align}\label{xm-log1}
        P_{m}^i(t) 
            &= r \int_0^t  P_{m}^i(s) \big[1- 2P_0(s) \big] ds+ \sigma\sum_{j} 
            \int_0^t e_j(s)\, P_{m^{i-}(j)}^i(s)\ ds,
    \end{align}
    
    where $m^{i-}(j)$ is the $i$-th vector with the $j$-th entry minus $1$.
    
    Finally, for the case with $|m| > 1$ and using the vectors from Table 5 in \cite{delgado2024simulating}, 
    $P_{m}^i$ satisfies
    \begin{align}\label{xm-log2}
        P_{m}^i(t) 
            &= r \int_0^t  P_{m}^i(s) \big[1- 2P_{0}(s) \big]  ds 
                + \sigma\sum_{j} \sqrt{m_j^i} 
            \int_0^t e_j(s)\, P_{m^-(j)}^i(s)\, ds,
    \end{align}
    where $m_j^i$ denotes the $j$-th entry of the $i$-th vector.
    
    We observe, as in the precedent example, that the functions  $P_{m}^i(t)$ with 
    $|m|\ge 1$ will be solved numerically using expressions \eqref{xm-log1} and 
    \eqref{xm-log2} in increasing order.

\subsection{Numerical examples}

We present a concise overview of the numerical findings for each example discussed in the previous section.

\subsubsection{Geometric Results}
\begin{center}
\begin{tabular}{ |c|c|c| } 
\hline
 Parameters & Initial condition & Final Estimation \\
\hline
 $\alpha=0.35 ,\sigma=0.17$ &$\tilde\alpha_0=0.4 ,\tilde\sigma=0.151048 $ & $\tilde\alpha_n=0.364751 ,\tilde\sigma_n=0.168444$ \\ 
 & $\tilde\alpha_0=0.2 ,\tilde\sigma=0.151048$ & $\tilde\alpha_n=0.364814 ,\tilde\sigma_n=0.156101$ \\ 
 & $\tilde\alpha_0=0.11 ,\tilde\sigma=0.151048$ & $\tilde\alpha_n=0.364836 ,\tilde\sigma_n=0.156445$ \\  
\hline
 $\alpha=0.63 ,\sigma=0.06$ &$\tilde\alpha_0=0.4 ,\tilde\sigma=0.056324 $ & $\tilde\alpha_n=0.638756 ,\tilde\sigma_n=0.0560588$ \\ 
 & $\tilde\alpha_0=0.2,\tilde\sigma=0.056324$ & $\tilde\alpha_n=0.641096 ,\tilde\sigma_n=0.0559717$ \\ 
 & $\tilde\alpha_0=0.8 ,\tilde\sigma=0.056324$ & $\tilde\alpha_n=0.638752 ,\tilde\sigma_n=0.055874$ \\  
\hline
$\alpha=0.9 ,\sigma=0.11$ &$\tilde\alpha_0=0.4 ,\tilde\sigma=0.10435 $ & $\tilde\alpha_n=0.917715 ,\tilde\sigma_n=0.0956571$ \\ 
 & $\tilde\alpha_0=0.25,\tilde\sigma=0.10435$ & $\tilde\alpha_n=0.987138 ,\tilde\sigma_n=0.0826281$ \\    
 & $\tilde\alpha_0=0.55 ,\tilde\sigma=0.10435$ & $\tilde\alpha_n=0.914632 ,\tilde\sigma_n=0.0972363$ \\   
\hline
\end{tabular}
\end{center}
\subsubsection{Logistic Results}

\begin{center}
\begin{tabular}{ |c|c|c| } 
\hline
 Parameters & Initial condition & Final Estimation \\
\hline
$\alpha=0.68 ,\sigma=0.09$ &$\tilde\alpha_0=0.55,\tilde\sigma=0.0810549  $ & $\tilde\alpha_n=0.673395 ,\tilde\sigma_n=0.0832877$ \\ 
 & $\tilde\alpha_0=0.95,\tilde\sigma=0.0810549 $ & $\tilde\alpha_n=0.675159 ,\tilde\sigma_n=0.0840031$ \\ 
 & $\tilde\alpha_0=0.12 ,\tilde\sigma=0.0810549 $ & $\tilde\alpha_n=0.674867 ,\tilde\sigma_n= 0.0898092 $ \\
\hline
 $\alpha=0.3 ,\sigma=0.05$ &$\tilde\alpha_0=1.01 ,\tilde\sigma=0.0443604  $ & $\tilde\alpha_n=0.319553 ,\tilde\sigma_n=0.0447231$ \\ 
 & $\tilde\alpha_0=0.1 ,\tilde\sigma=0.0443604 $ & $\tilde\alpha_n=0.302488 ,\tilde\sigma_n=0.096654$ \\ 
 & $\tilde\alpha_0=0.5 ,\tilde\sigma=0.0443604 $ & $\tilde\alpha_n=0.302418 ,\tilde\sigma_n= 0.0646823 $ \\ 
\hline
$\alpha=1.2 ,\sigma=0.2$ &$\tilde\alpha_0=0.57 ,\tilde\sigma=  0.179743  $ & $\tilde\alpha_n=1.21095 ,\tilde\sigma_n=0.256058$ \\ 
 & $\tilde\alpha_0=0.67 ,\tilde\sigma=  0.179743 $ & $\tilde\alpha_n=1.21411 ,\tilde\sigma_n=0.243503$ \\ 
 & $\tilde\alpha_0=0.95 ,\tilde\sigma=  0.179743 $ & $\tilde\alpha_n=1.21081 ,\tilde\sigma_n= 0.235354 $ \\ 
\hline
\end{tabular}
\end{center}

\begin{figure}[h]
\caption{Convergence analysis of the Gradient Descent algorithm for estimating the drift parameter $\alpha$ in the Geometric Brownian Motion and the OU process. The horizontal line indicates the ground truth value ($\alpha = 0.9$ for the GMB and $\alpha = 1.7$ for the OU). Distinct markers represent the evolution of the parameter estimate starting from different 5 initial guesses. As shown, all trajectories converge rapidly toward the true value within approximately 6 iterations, demonstrating the robustness and stability of the Wiener Chaos-based optimization framework.}
\centering
\includegraphics[scale=0.35]{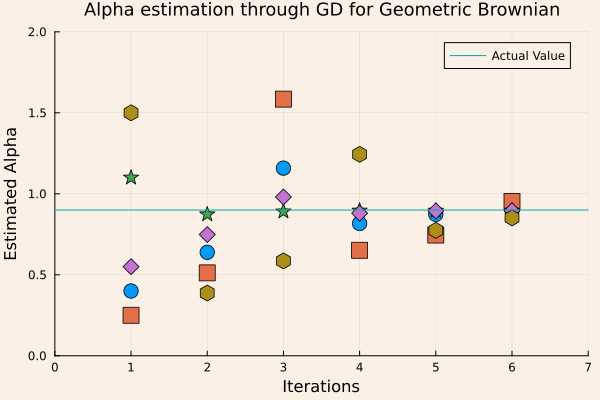}
\includegraphics[scale=0.35]{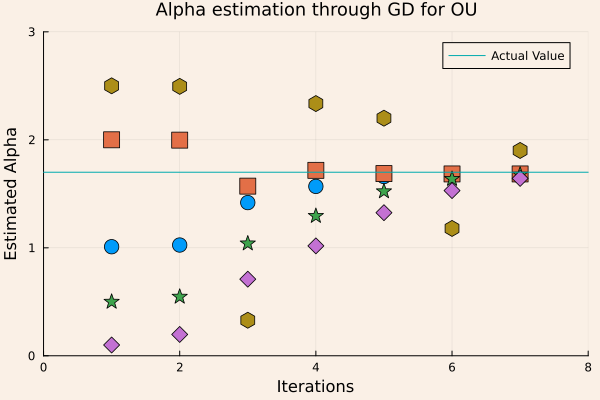}
\end{figure}

\section{Application to Empirical Data}\label{real-data-app}

In this section, we demonstrate the practical applicability and versatility of our parameter estimation framework by utilizing two distinct real-world datasets. First, we validate the method on financial market data using a canonical linear model, the Geometric Brownian Motion (GBM). Subsequently, we apply the algorithm to characterize individual biological growth, employing the non-linear stochastic logistic differential equation. This dual application highlights the algorithm's robustness across different domains, scaling, and underlying stochastic dynamics.

\subsection{Empirical Validation: S\&P 500 Index}

To assess the practical applicability and robust performance of the proposed stochastic gradient descent (SGD) approach coupled with the Wiener Chaos Expansion (WCE), we initially applied the methodology to empirical financial market data. The dataset selected consists of daily adjusted closing prices for the S\&P 500 index ($\hat{}$GSPC), representing a proxy for the broad US equity market.

Based on quantitative finance standards, the underlying dynamics were modeled using a Geometric Brownian Motion (GBM) specification. The observed period spans a full calendar year, from January 1, 2017, to December 31, 2017. This interval was selected to provide a sufficiently noisy real-world scenario to test the algorithm's learning capabilities, yet remain within a standard stable GBM regime.

Consistent with the procedure used for synthetic data, the volatility parameter ($\sigma$) was pre-estimated analytically via standard Maximum Likelihood Estimation (MLE) of the log-returns. Subsequently, the WCE-SGD algorithm was deployed to learn the drift parameter ($\mu$) by minimizing the functional error between the expansion propagator and the empirical trajectory.

The visual results of this validation are presented in Figure \ref{fig:sp500_wce}. As illustrated, the methodology yields an accurate reconstruction of the market dynamics. The empirical trajectory of the S\&P 500 index is remarkably contained within the 95\% confidence bands generated by the simulations based on the WCE-learned parameters. Furthermore, the numerical convergence of $\mu$ showed excellent agreement with the benchmark MLE value ($\hat{\mu} = 0.1589636$ with the WCE, and $\hat{\mu}_{MLE} = 0.1725922$). Consequently, the method is considered to function with high precision on contrasted real-world datasets, successfully capturing the stochastic and deterministic components of the underlying linear process.

\begin{figure}[htbp]
    \centering
    \includegraphics[width=0.8\textwidth]{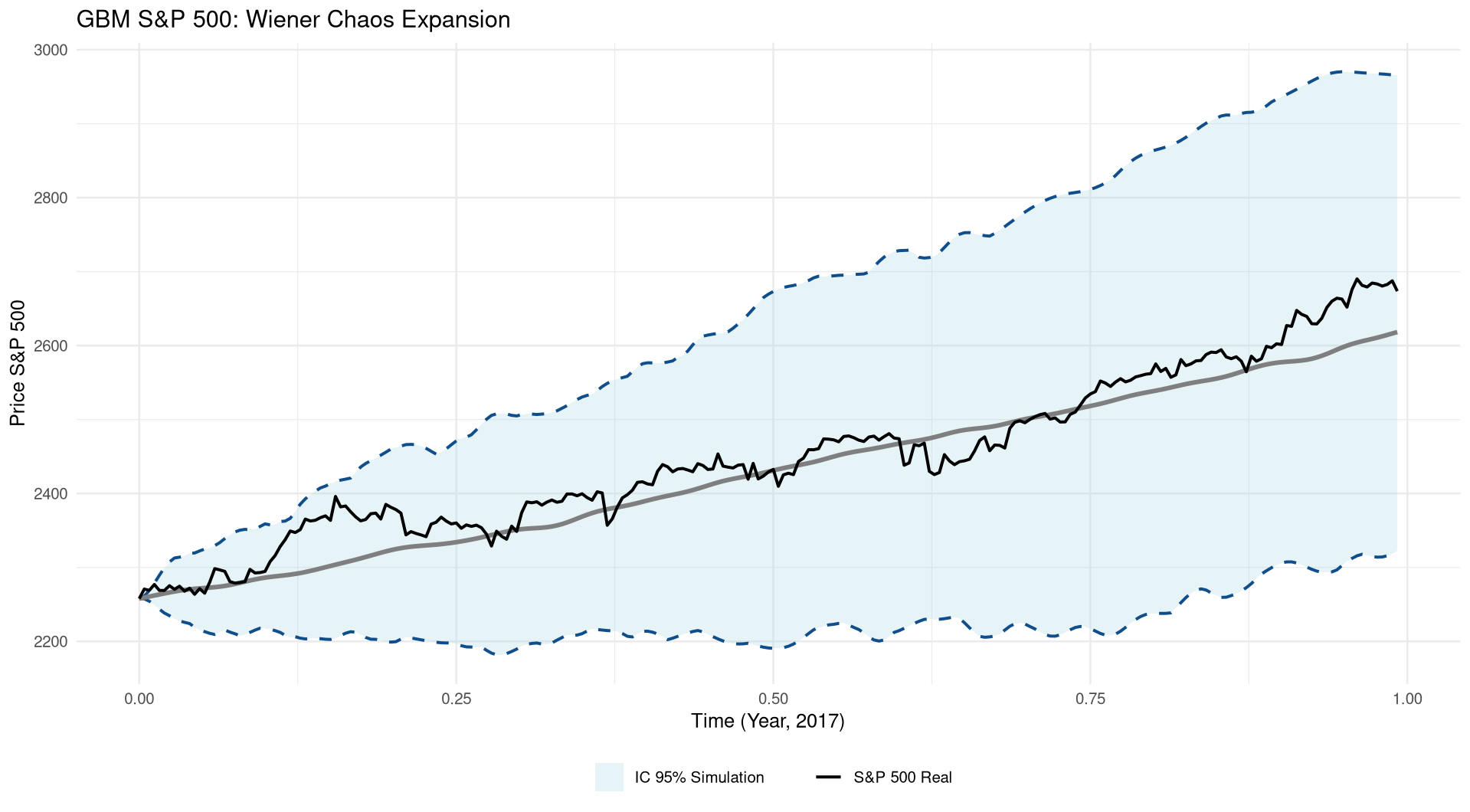}
    \caption{\textbf{GBM Parameter Learning Validation on S\&P 500 Index Data (2017).} Comparison between the daily adjusted closing prices (solid black line) and stochastic simulations based on parameters learned via the proposed WCE-SGD method ($\hat{\mu} \approx 0.1590$, $\hat{\sigma} \approx 0.0674$). The solid gray line represents the mean of 100 simulations, while the shaded blue area indicates the 95\% empirical confidence interval. The close fit demonstrates the successful application of the methodology to noisy empirical financial data.}
    \label{fig:sp500_wce}
\end{figure}

\subsection{Empirical Validation: Microbial Growth Dynamics}

Having validated the approach on a linear model, we now extend the application to a highly non-linear system where multi-trajectory data is available.

\subsubsection{Data Acquisition and Scientific Relevance}
The biological dataset, integrated into the \texttt{growthcurver} R package,\cite{sprouffske2016growthcurver, R-growthcurver}, comprises high-resolution microbial growth curves obtained via automated spectrophotometry. The raw data were recorded using a 96-well microtiter plate reader, which measures the Optical Density at 600 nm ($OD_{600}$) at regular intervals over a 24-hour period. This experimental setup allows for the simultaneous monitoring of multiple replicates, capturing the complete life cycle of the microbial population, including the lag, exponential, and stationary phases.

The primary utility of this dataset is to provide a robust framework for the computational modeling of population dynamics. By applying non-linear regression models (such as the logistic growth equation), researchers can derive critical kinetic parameters, namely the carrying capacity ($K$), the intrinsic growth rate ($\alpha$), and the initial population size ($P_0$).

\subsubsection{Parameter Estimation via WCE-SGD}
Modeling this biological phenomenon requires addressing the intrinsic multiplicative noise and the non-linear saturation limit ($K$). To isolate the stochastic dynamics from the spatial scale, we implemented a hybrid estimation strategy.

\begin{figure}[htbp]
    \centering
    \includegraphics[width=0.8\textwidth]{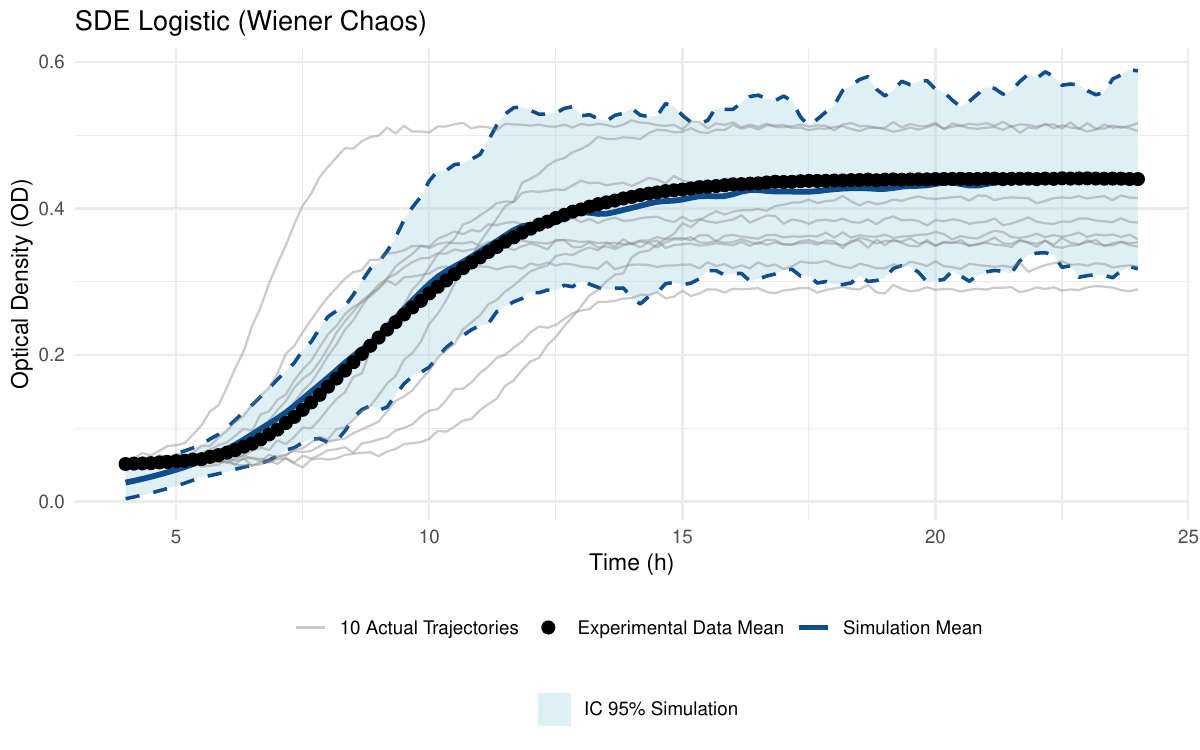}
    \caption{\textbf{Stochastic Logistic Parameter Learning on Microbial Growth Data.} Comparison between empirical optical density trajectories (gray lines) from a 96-well microtiter plate and stochastic simulations driven by parameters learned via the WCE-SGD framework. The data was pre-standardized to estimate the carrying capacity $K$. The methodology accurately captures both the asymmetrical, non-linear drift of the exponential phase and the inherent biological noise, encapsulating the individual trajectories within the 95\% confidence bands (shaded area).}
    \label{fig:logistic_wce}
\end{figure}

First, a deterministic standardization was performed. The cross-well mean trajectory of the raw optical density data was computed, and a standard Levenberg-Marquardt Non-Linear Least Squares (NLS) algorithm was utilized to estimate the empirical carrying capacity ($\hat{K}$) and the initial population state. The entire dataset was then normalized by dividing it by $\hat{K}$, mapping the dynamic state variable $X_t$ strictly to the domain $[0,1]$, which mathematically stabilizes the stochastic logistic equation $dX_t = \alpha X_t(1-X_t)dt + \sigma X_t dW_t$.

Second, the volatility parameter ($\sigma$) was evaluated analytically by computing the mean quadratic variation across all individual normalized well trajectories. This captures the true experimental stochasticity. Finally, the WCE-SGD algorithm was deployed on the normalized mean trajectory to estimate the intrinsic growth rate ($\alpha$). To prevent numerical instability during optimization, the Barzilai-Borwein step size was constrained using dynamic dampening (gradient clipping), ensuring smooth convergence to the global minimum despite the non-linear topology of the objective functional.

The outcomes, presented in Figure \ref{fig:logistic_wce}, confirm that the estimated parameters successfully reconstruct the experimental behavior. By re-scaling the simulations back to the original biological limits ($K$), the model accurately reflects both the deterministic macroscopic trend (the sigmoidal drift) and the microscopic biological variance observed across the wells.

\section{Concluding Remarks}

In this work, we have developed and validated a novel computational framework for the inverse problem of parameter estimation in Stochastic Differential Equations (SDEs). The core innovation of our approach lies in the integration of the Wiener Chaos Expansion (WCE) with Stochastic Gradient Descent (SGD), which effectively transforms the challenge of stochastic parameter learning into a deterministic optimization problem.

By projecting the solution of the SDE onto the Wiener chaos space, we decomposed the stochastic trajectory into a system of deterministic ordinary differential equations, referred to as the \textit{propagator}. This spectral decomposition allowed us to construct a robust objective function that simultaneously targets the deterministic trend (via the zeroth-order chaos mode) and the stochastic fluctuations (via higher-order modes and quadratic variation). Consequently, we were able to decouple the estimation of drift and diffusion parameters, addressing a common numerical difficulty in stochastic inverse problems.

The numerical experiments conducted on a diverse set of synthetic benchmarks (including the Ornstein-Uhlenbeck process, Geometric Brownian Motion, and the Stochastic Logistic Equation) demonstrated the algorithm's versatility and convergence properties. In all cases, the proposed method successfully recovered the underlying parameters with high accuracy, even in the presence of multiplicative noise and non-linearities.

Crucially, the extension of this framework to empirical data underscores its practical utility for modeling complex real-world systems. By accurately identifying the drift and volatility components of the S\&P 500 index within a real-market context, the WCE-SGD approach proved its efficacy in capturing financial market stochasticity. Furthermore, its application to high-resolution microbial growth data demonstrated the method's capacity to handle non-linear biological saturation limits and intrinsic experimental variance using the stochastic logistic model. These dual validations confirm that the methodology represents a significant theoretical advancement in spectral decomposition and offers a computationally efficient alternative for parameter learning. While its performance is demonstrated here through specific examples in ecology, biology, and finance, the underlying framework is fundamentally versatile and applicable across a broad spectrum of scientific disciplines where stochastic modeling is required.

Future research will prioritize the extension of this methodology to a broader range of high-frequency financial instruments, further refining the coupling between the WCE and SGD under volatile market conditions. Additionally, we aim to develop a robust numerical implementation of the Stochastic Partial Differential Equation (SPDE) framework introduced in Section \ref{SPDE-WCE}. 

Another key avenue of exploration involves the application of this approach to higher-dimensional systems and coupled SDEs. In such contexts, the inherent scalability of the SGD-WCE framework is expected to offer substantial computational advantages and improved tractability over traditional Maximum Likelihood Estimation (MLE) or computationally intensive Bayesian inference methods.

\section{Software and Codes}
The numerical experiments have been implemented in 
on an average laptop (ASUS Vivobook X512DA with
Processor AMD Ryzen 5 3500U and Radeon Graphics, 1200 Mhz, 4 Cores, 8 Logical Processors). The code can be
found under the following link:\\
\hyperlink{https://github.com/JJulianPavon/GradientDescentSDE}{https://github.com/JJulianPavon/GradientDescentSDE}

\appendix

\section{Background on Malliavin Calculus and Wiener-chaos expansion}\label{WienerChaos}

For the sake of completeness, this section reviews the fundamental concepts of Malliavin calculus and the Wiener chaos expansion utilized in this work. For a comprehensive treatment, we refer the reader to the classical book of  \cite{nualart2006malliavin}.

\subsection{Gaussian Probability Space and Wiener Chaos}
Let $\mathfrak{H} = L^2([0,T])$ be a separable Hilbert space with inner product $\langle \cdot, \cdot \rangle_{\mathfrak{H}}$ and norm $\|\cdot\|_{\mathfrak{H}}$. We consider an isonormal Gaussian process $W = \{W(h) : h \in \mathfrak{H}\}$ defined on a complete probability space $(\Omega, \mathcal{F}, \mathbb{P})$. This implies that $W$ is a centered Gaussian family of random variables satisfying the isometry property:
\begin{equation}
    \mathbb{E}[W(h)W(g)] = \langle h, g \rangle_{\mathfrak{H}}, \quad \forall h, g \in \mathfrak{H}.
\end{equation}
We assume that $\mathcal{F}$ is the $\sigma$-algebra generated by $W$. In our specific setting, $W(h)$ is the Wiener integral $W(h) = \int_0^T h(t) dW_t$, where $W_t$ is a standard Brownian motion.

To construct an orthogonal basis for $L^2(\Omega, \mathcal{F}, \mathbb{P})$, we employ the Hermite polynomials defined by:
\begin{equation}
    H_n(x) = \frac{(-1)^n}{\sqrt{n!}} e^{x^2/2} \frac{d^n}{dx^n} \left( e^{-x^2/2} \right), \quad n \geq 0,
\end{equation}
with $H_0(x)=1$. Let $\{e_i\}_{i \geq 1}$ be an orthonormal basis of $\mathfrak{H}$. We define the set of multi-indices of finite length as:
\begin{equation}
    \mathcal{J} = \{ \alpha = (\alpha_1, \alpha_2, \dots) : \alpha_i \in \mathbb{N} \cup \{0\}, \ |\alpha| := \sum_{i=1}^\infty \alpha_i < \infty \}.
\end{equation}
For any multi-index $\alpha \in \mathcal{J}$, we define the Fourier-Hermite functional $\xi_\alpha$ as:
\begin{equation} \label{eq:xi_alpha}
    \xi_\alpha = \prod_{i=1}^\infty H_{\alpha_i}(W(e_i)).
\end{equation}
It is a classical result (see \cite[Prop 1.1.1]{nualart2006malliavin}) that the family $\{ \xi_\alpha : \alpha \in \mathcal{J} \}$ forms a complete orthonormal basis for $L^2(\Omega, \mathcal{F}, \mathbb{P})$. Consequently, any square-integrable random variable $F \in L^2(\Omega)$ admits a unique \textit{Wiener Chaos Expansion}:
\begin{equation} \label{eq:WCE}
    F = \sum_{\alpha \in \mathcal{J}} F_\alpha \xi_\alpha, \quad \text{where} \quad F_\alpha = \mathbb{E}[F \xi_\alpha].
\end{equation}
The subspace spanned by $\{\xi_\alpha : |\alpha|=n\}$ is called the $n$-th Wiener chaos $\mathcal{H}_n$, and we have the decomposition $L^2(\Omega) = \bigoplus_{n=0}^\infty \mathcal{H}_n$.

\begin{remark}\label{1st-2nd_moments}
Useful properties of this decomposition include:
\begin{enumerate}
    \item \textbf{Expectation:} $\mathbb{E}[F] = F_0$ (since $\mathbb{E}[\xi_\alpha] = 0$ for $|\alpha| > 0$).
    \item \textbf{Isometry:} $\mathbb{E}[F^2] = \sum_{\alpha \in \mathcal{J}} F_\alpha^2$.
    \item \textbf{Variance:} $\text{Var}(F) = \sum_{\{\alpha : |\alpha| \geq 1\}} F_\alpha^2$.
\end{enumerate}
\end{remark}

\subsection{The Malliavin Derivative}
Let $\mathcal{S}$ denote the set of smooth random variables of the form:
\begin{equation}
    F = f(W(h_1), \dots, W(h_n)),
\end{equation}
where $f \in C_p^\infty(\mathbb{R}^n)$ (smooth functions with polynomial growth) and $h_i \in \mathfrak{H}$. The \textit{Malliavin derivative} $DF$ is defined as the $\mathfrak{H}$-valued random variable:
\begin{equation}
    DF = \sum_{i=1}^n \partial_i f(W(h_1), \dots, W(h_n)) h_i.
\end{equation}
We denote by $D_t F$ the realization of $DF$ at time $t$. The operator $D$ is closable from $L^p(\Omega)$ to $L^p(\Omega; \mathfrak{H})$. We define the Sobolev space $\mathbb{D}^{1,2}$ as the closure of $\mathcal{S}$ under the norm:
\begin{equation}
    \|F\|_{1,2}^2 = \mathbb{E}[|F|^2] + \mathbb{E}[\|DF\|_{\mathfrak{H}}^2].
\end{equation}
Crucially, the derivative of a basis element $\xi_\alpha$ acts as a lowering operator:
\begin{equation} \label{eq:D_xi}
    D_t \xi_\alpha = \sum_{i=1}^\infty \sqrt{\alpha_i} e_i(t) \xi_{\alpha^-(i)},
\end{equation}
where $\alpha^-(i)$ is the multi-index with the $i$-th component reduced by one: 
\begin{equation}\label{m-i}
\alpha^-(i):= (\alpha_1,\ldots,\alpha_{i-1},\alpha_i-1,\alpha_{i+1},\ldots)\quad\mbox{ if } \alpha_{i}\ge 1,
\end{equation}
 and is undefined otherwise.

\subsection{The Skorohod Integral}
The divergence operator $\delta$, also known as the \textit{Skorohod integral}, is the adjoint of the Malliavin derivative $D$. Its domain, $\text{Dom}(\delta) \subset L^2(\Omega; \mathfrak{H})$, consists of processes $u$ such that:
 $   |\mathbb{E}[\langle DF, u \rangle_{\mathfrak{H}}]| \leq C \|F\|_{L^2(\Omega)}, \quad \forall F \in \mathbb{D}^{1,2}$.
For $u \in \text{Dom}(\delta)$, the random variable $\delta(u)$ is characterized by the duality relationship:
\begin{equation} \label{def-delta}
    \mathbb{E}[F \delta(u)] = \mathbb{E}[\langle DF, u \rangle_{\mathfrak{H}}], \quad \forall F \in \mathbb{D}^{1,2}.
\end{equation}
A fundamental result (see \cite{nualart2006malliavin}) connects the Skorohod integral to the classical It\^o integral: if a process $u \in L^2(\Omega \times [0,T])$ is adapted to the filtration generated by the Brownian motion, then $u \in \text{Dom}(\delta)$ and $\delta(u)$ coincides with the It\^o integral:
\begin{equation}
    \delta(u) = \int_0^T u(t) dW_t.
\end{equation}
This connection allows us to derive the Wiener chaos expansion for solutions of SDEs. Specifically, using the duality property \eqref{def-delta} and the derivative formula \eqref{eq:D_xi}, the projection of a stochastic integral onto the chaos mode $\alpha$ is given by:
\begin{equation} \label{eq:chaos_prop}
    \mathbb{E}\left[ \left( \int_0^t u(s) dW_s \right) \xi_\alpha \right] = \mathbb{E} \int_0^t u(s) D_s \xi_\alpha ds = \sum_{i=1}^\infty \sqrt{\alpha_i} \int_0^t \mathbb{E}[u(s) \xi_{\alpha^-(i)}] e_i(s) ds.
\end{equation}
This identity is the cornerstone for deriving the deterministic propagator system used in this paper.

\bibliographystyle{cas-model2-names}
\bibliography{references}

\end{document}